\documentclass{article} 

\usepackage{iclr2027_conference}
\usepackage{times}

\usepackage{newtxtt} 
\usepackage[scaled=1]{biolinum} 
\usepackage{graphicx}
\usepackage{booktabs} 
\usepackage[font={footnotesize}]{caption}
\usepackage{subcaption} 
\usepackage{float} 
\usepackage[utf8]{inputenc} 
\usepackage[T1]{fontenc}    
\usepackage[colorlinks=true, citecolor=Violet, linkcolor=red, urlcolor=magenta]{hyperref}       
\usepackage{url}            
\usepackage{booktabs}   

\usepackage{amsfonts}       
\usepackage{nicefrac}       
\usepackage{microtype}      
\usepackage{lipsum}
\usepackage{fancyhdr}       
\usepackage{wrapfig}
\usepackage{colortbl}
\usepackage[dvipsnames]{xcolor}
\usepackage{graphicx}    
\graphicspath{{media/}}     
\usepackage[most]{tcolorbox}
\usepackage{alltt}
\usepackage{enumitem}
\usepackage{bm}
\usepackage{bbm}
\usepackage{tikz}
\usetikzlibrary{arrows.meta}

\usepackage{amsmath}
\usepackage{amssymb}
\usepackage{mathtools}
\usepackage{amsthm}
\usepackage{amsfonts} 
\usepackage{thmtools, thm-restate} 

\usepackage[capitalize,noabbrev]{cleveref}

\usepackage[vlined,ruled,linesnumbered]{algorithm2e}
\usepackage{adjustbox}
\usepackage{multirow,mathtools}
\renewcommand{\arraystretch}{1.15}

\usepackage{etoc}

\usepackage{threeparttable}
\usepackage{arydshln} 

\theoremstyle{plain}
\newtheorem{theorem}{Theorem}[section]

\theoremstyle{definition}
\newtheorem{definition}[theorem]{Definition}

\theoremstyle{remark}

\newtcbox{\hlprimarytab}{on line, box align=base, colback=BlueGreen!20,colframe=blue,size=fbox,arc=3pt, before upper=\strut, top=-2.5pt, bottom=-4.5pt, left=-2pt, right=-2pt, boxrule=0pt}
\newtcbox{\hlsecondarytab}{on line, box align=base, colback=WildStrawberry!10,colframe=orange,size=fbox,arc=3pt, before upper=\strut, top=-2.5pt, bottom=-4.5pt, left=-2pt, right=-2pt, boxrule=0pt}

\newcommand{\refs}[2]{\hyperref[#1]{\ref*{#1}#2}}

\newcommand{\tofu}{\texttt{TOFU}\xspace}
\newcommand{\muse}{\texttt{MUSE}\xspace}

\newcommand{\wmdp}{\texttt{WMDP}\xspace}

\newcommand{\openunlearning}{\texttt{OpenUnlearning}\xspace}

\newtcolorbox[auto counter, number within=section]{mydefinition}[2][]{%
  enhanced,
  breakable,
  colback=green!1,
  colframe=green!40!black,
  colbacktitle=green!10!white,
  coltitle=black,
  fonttitle=\bfseries\small,
  fontupper=\small,
  title={Definition~\thetcbcounter: #2},
  label={def:#2},
  attach boxed title to top left={yshift=-3.25mm, xshift=2mm},
  boxed title style={colframe=green!40!black},
  left=10pt,
  right=10pt,
  top=9pt,
  bottom=3pt,
  arc=1.5pt,
  drop shadow=black!25,
  #1
}

\newtcolorbox[auto counter, number within=section]{mythm}[2][]{%
  enhanced,
  breakable,
  colback=cyan!1, 
  colframe=cyan!40!black, 
  colbacktitle=cyan!10!white, 
  coltitle=black, 
  fonttitle=\bfseries\small,
  fontupper=\small, 
  title={Theorem~\thetcbcounter: #2}, 
  label={thm:#2},
  attach boxed title to top left={yshift=-3.25mm, xshift=2mm},
  boxed title style={colframe=cyan!35!black}, 
  left=10pt, 
  right=10pt, 
  top=9pt, 
  bottom=3pt, 
  arc=1.5pt, 
  drop shadow=black!25, 
  #1
}

\usepackage{comment} 

\newcommand{\xbf}{{\mathbf x}}

\newcommand{\ybf}{{\mathbf y}}

\newcommand{\btheta}{{\bm{\theta}}}

\definecolor{gradblue}{RGB}{31,119,180}
\definecolor{noiseorange}{RGB}{214,96,27}
\definecolor{coveragepurple}{RGB}{117,107,177}
\definecolor{controlgreen}{RGB}{27,158,119}

\title{Preemptive LLM Unlearning against Forbidden Capability Acquisition via Gradient Sealing}

\author{Kemou Li$^{1}$ \hspace{0.2cm} 
Qizhou Wang$^{2}$ \hspace{0.2cm}
Yue Wang$^{3}$ \hspace{0.2cm}
Fengpeng Li$^{4}$ \hspace{0.2cm}
Zhuan Shi$^{5, 6}$ \\
\textbf{Negar Rostamzadeh}$^{5, 6}$ \hspace{0.2cm} 
\textbf{Golnoosh Farnadi}$^{5, 6}$ \hspace{0.2cm} 
\textbf{Masashi Sugiyama}$^{2, 7}$ \hspace{0.2cm}
\textbf{Jiantao Zhou}$^1$\vspace{1.5mm}\\
$^1$State Key Laboratory of Internet of Things for Smart City, University of Macau\\
$^2$RIKEN Center for Advanced Intelligence Project \hspace{0.25cm} 
$^3$The University of Melbourne \hspace{0.25cm} \\
$^4$King Abdullah University of Science and Technology \hspace{0.25cm}  
$^5$Mila -- Québec AI Institute \\
$^6$McGill University \hspace{0.25cm}
$^7$The University of Tokyo
}

\iclrfinalcopy 
\begin{document}

\etocdepthtag.toc{mtchapter}
\maketitle

\begin{abstract}
Open-weight LLMs are released not only as fixed products but also as substrates for downstream fine-tuning. 
This openness, however, creates legal and ethical risks because users may misuse fine-tuning to instill illicit knowledge or enable hostile operations.
Model providers therefore need apre-release defense against such acquisition, motivating the problem of \emph{preemptive unlearning}. 
Unlike \emph{retrospective unlearning}, which removes capabilities already present in a fixed model, preemptive unlearning seeks to prevent their acquisition under unseen attack data and future fine-tuning procedures.
Despite its practical importance, this setting remains largely unexplored, presents distinct challenges, and is therefore the central focus of our work. 
We first verify that existing retrospective methods provide insufficient pre-release protection.
Even when forbidden capabilities are suppressed in current outputs, forbidden-domain data can still induce gradients through internal pathways, enabling later acquisition. 
Motivated by this finding, we propose a \emph{gradient-sealing principle} that blocks these pathways by pushing relevant pre-activations into the negative region, where ReLU-family activations exhibit zero or near-zero derivatives. 
Experiments across multiple LLM families demonstrate our stronger resistance to downstream acquisition than retrospective baselines, validating gradient sealing as an effective mechanism for pre-release protection.
\end{abstract}

\section{Introduction}
Open-weight large language models (LLMs)~\citep{jiang2023mistral,dubey2024llama,gemma2025gemma3,yang2025qwen3} have become central to modern AI research and practice. 
Their accessible parameters enable cost-effective, privacy-preserving deployment and domain-specific adaptation with modest compute.
However, the same openness creates legal and ethical risks, as malicious users may fine-tune released models on harmful data to acquire capabilities for disinformation, explicit content generation, dangerous operational guidance, or privacy violations~\citep{qi2024finetuning,rosati2024representation,kaunismaa2026eliciting}, threatening both individuals and society.
Model providers therefore need a \emph{pre-release defense} that designates a \emph{forbidden domain} and makes the corresponding capabilities resistant to acquisition after release.
This motivates a critical yet underexplored question:
\begin{center}
\vspace{-0.15mm}
\textit{Can we precondition a model to resist the acquisition of\\ forbidden-domain capabilities through downstream fine-tuning?}
\end{center}

A closely related line of research is LLM unlearning~\citep{yao2024large,li2024wmdp,zhang2024negative}. 
Given a designated forget set, unlearning methods modify model parameters to suppress content already learned, as shown in Fig.~\refs{fig:pipeline}{a}~\citep{jang2023knowledge}. 
However, these methods are inherently \textit{retrospective}. 
Our pre-release defense is instead \textit{preemptive}: rather than removing knowledge already embedded in the model, it aims to make forbidden-domain capabilities resistant to acquisition through downstream fine-tuning, as shown in Fig.~\refs{fig:pipeline}{b}. 
Preemptive unlearning remains largely underexplored and particularly challenging because users have full access to released model parameters, making external  guardrails~\citep{inan2023llama,han2024wildguard} extremely difficult to enforce.

We further distinguish {retrospective} from {preemptive} unlearning by showing why retrospective methods provide insufficient pre-release protection. 
The key insight is that existing retrospective methods typically suppress current outputs while leaving the underlying gradients unconstrained~\citep{fan2025towards}.
Because downstream fine-tuning updates model parameters through these gradients, such open gradient pathways can still enable the acquisition of forbidden-domain capabilities. 
Our theoretical analysis in \S\ref{sec:theoretical_motivation} formalizes this intuition, showing that forbidden-domain data can induce gradients along domain-sensitive pathways even when the released model exhibits little forbidden-domain capability. 
Empirical results in \S\ref{sec:empirical_motivation} further show that low release capability can hide high acquisition susceptibility, while proxy look-ahead can locate and targeted attenuation can weaken those responsible pathways.
Together, these findings reinforce known robustness limitations of retrospective unlearning~\citep{wang2025invariance,lang2026downgrade}.

\begin{figure}
    \centering
    \includegraphics[width=\linewidth]{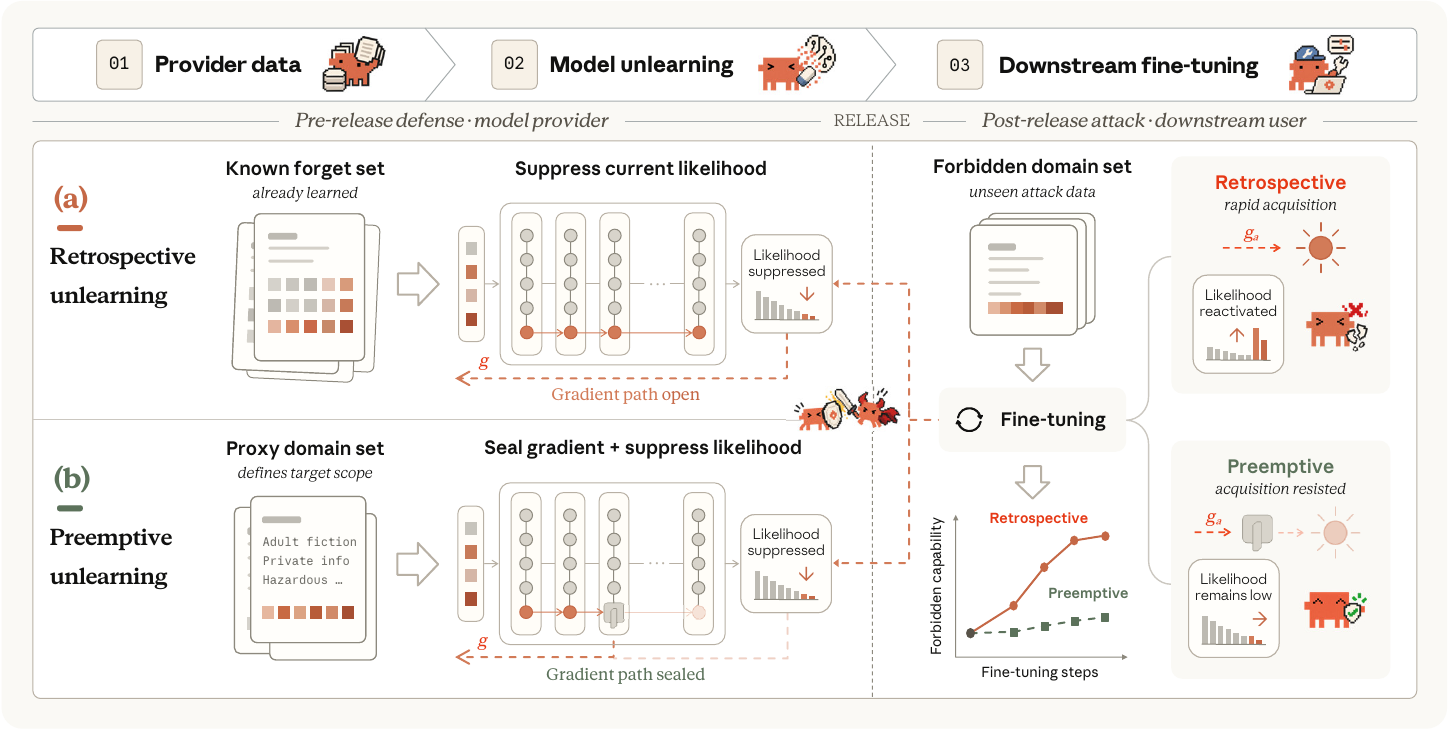}
    \caption{\textbf{Retrospective unlearning vs. preemptive unlearning.}
(a) Retrospective unlearning suppresses an already learned capability but may leave its gradient pathway open.
(b) Preemptive unlearning instead uses a forbidden-domain proxy set before release to seal that pathway, preventing future acquisition from fine-tuning.}
    \label{fig:pipeline}
\end{figure}

These findings suggest that pre-release defense must seal acquisition pathways beyond output suppression.
Directly penalizing acquisition-gradient norms, however, requires costly second-order differentiation.
We instead exploit ReLU-family activation gates with zero or near-zero derivatives for sufficiently negative inputs.
Pushing selected pre-activations into this regime attenuates their local gradient factors using only first-order updates.
Accordingly, we propose \textbf{Gradient-Sealed Unlearning (GSU)}, which uses proxy learning on a disposable copy to identify acquisition-sensitive gates from upward pre-activation shifts under matched contexts.
Restarting from the original weights, GSU jointly optimizes a \emph{seal loss}, response suppression, and retain supervision.
The seal loss in Eq.~\eqref{eq:seal_loss} pushes selected forbidden-domain pre-activations below a shared negative-tail threshold, while the other terms suppress release-time responses and preserve normal utility.

To our knowledge, GSU represents one of the earliest systematic efforts to formulate and evaluate preemptive unlearning as a pre-release defense.
It may also help strengthen retrospective unlearning against post-unlearning attacks~\citep{hu2025unlearning}.
In \S\ref{sec:experiments}, we extensively evaluate GSU on \wmdp~\citep{li2024wmdp} and \tofu~\citep{maini2024tofu} across multiple LLM families, where it shows stronger resistance to downstream acquisition of forbidden capabilities than competitive retrospective unlearning baselines.
These results support GSU's gradient-sealing principle and activation-gating mechanism.

\section{Problem Statement}
\label{sec:formulation}

Open-weight LLMs face legal and ethical risks similar to those of closed-weight LLMs~\citep{bommasani2021opportunities,kapoor2024societal}. 
Releasing model parameters introduces an additional challenge: users can fine-tune a released model to acquire a designated forbidden capability, even after retrospective unlearning ~\citep{yao2024large,wang2025rethinking}, as discussed in \S\ref{sec:motivation}. 
This motivates our study of \emph{preemptive unlearning}, an equally important yet largely unexplored setting. 
Unlike retrospective unlearning, which removes capabilities already present in a model, preemptive unlearning conditions a model before release to resist the acquisition of forbidden capabilities through future fine-tuning.

Let \(\mathcal F\) denote the forbidden domain and \(\mathcal R\) the normal domain whose utility should be preserved, with corresponding distributions \(\mathbb P_{\mathcal F}\) and \(\mathbb P_{\mathcal R}\) over prompt--response pairs.
For a performance functional $\mathsf{Perf}$ and model parameters $\btheta$, define $\mathsf{Perf}_{\mathcal F}(\btheta) \coloneq \mathsf{Perf}(\btheta;\mathbb P_{\mathcal F})$ and $\mathsf{Perf}_{\mathcal R}(\btheta) \coloneq \mathsf{Perf}(\btheta;\mathbb P_{\mathcal R})$, where higher values indicate stronger forbidden-domain capability and normal-domain utility, respectively.
We formalize preemptive unlearning through the following attack--defense framework.

\subsection{Post-Release Attack}
After release, an attacker obtains white-box access to the released model parameters $\btheta_{\mathrm{rel}}$ and an unseen forbidden-domain dataset $\mathcal D_a\sim\mathbb P_{\mathcal F}^{\otimes n_a}$.
An attack $\mathcal{A}$ from a class $\mathfrak A$ produces $\btheta_{\mathrm{atk}}=\mathcal A(\btheta_{\mathrm{rel}};\mathcal D_a)$.
Within $\mathfrak{A}$, the attacker may choose the training objective, optimizer, schedule, and other configurations, with the goal of acquiring the forbidden capability while retaining normal utility, as formalized below.

\begin{definition}[Successful acquisition attack]
\label{def:successful_attack}
Given a maximum acceptable forbidden-capability threshold $\epsilon_f$ and a minimum usable-utility threshold $\epsilon_u$, an attack $\mathcal A\in\mathfrak A$ is successful against $\btheta_{\mathrm{rel}}$ if the resulting $\btheta_{\mathrm{atk}}$ satisfies both (i)~$\mathbb E[\mathsf{Perf}_{\mathcal F}(\btheta_{\mathrm{atk}})]>\epsilon_f$ and (ii)~$\mathbb E[\mathsf{Perf}_{\mathcal R}(\btheta_{\mathrm{atk}})]>\epsilon_u$.
The expectations are taken over the sampling of $\mathcal D_a$ and the internal randomness of $\mathcal A$.
\end{definition}

Def.~\ref{def:successful_attack} excludes attacks that acquire the forbidden capability only by severely degrading normal-domain utility.
In such cases, high forbidden-domain performance may result primarily from fitting $\mathcal D_a$ rather than from reusing information or capabilities retained in $\btheta_{\mathrm{rel}}$.
Such attacks therefore fall outside our intended scope and do not constitute a failure of the defense introduced below.

\subsection{Pre-Release Defense}
Before release, the defender has access to the original trained model parameters $\btheta_o$ and a forbidden-domain proxy set $\mathcal D_f\sim\mathbb P_{\mathcal F}^{\otimes n_f}$. 
Remaining agnostic to the future attack dataset \(\mathcal D_a\) and attack procedure \(\mathcal A\), we seek a reliable preemptive unlearning algorithm $\mathcal U$ that produces defense-enhanced parameters $\btheta_{\mathrm{rel}}=\mathcal U(\btheta_o;\mathcal D_f)$ resistant to every \(\mathcal A\in\mathfrak A\) applied using any relevant \(\mathcal D_a\).

\begin{definition}[Preemptive unlearning]
\label{def:preemptive_unlearning}
Fix an attack class $\mathfrak A$, thresholds $(\epsilon_f,\epsilon_u)$ from Def.~\ref{def:successful_attack}, and an allowed release-time utility degradation $\epsilon_r\geq0$ such that $\epsilon_u<\mathsf{Perf}_{\mathcal R}(\btheta_o)-\epsilon_r$.
A released model $\btheta_{\mathrm{rel}}$ is $(\epsilon_f,\epsilon_r,\epsilon_u)$-preemptively unlearned against $\mathfrak A$ w.r.t. $(\mathcal F,\mathcal R)$ if (i) no attack $\mathcal A\in\mathfrak A$ is successful against $\btheta_{\mathrm{rel}}$ under Def.~\ref{def:successful_attack}, and (ii) $\btheta_{\mathrm{rel}}$ satisfies $\mathsf{Perf}_{\mathcal R}(\btheta_{\mathrm{rel}})\geq\mathsf{Perf}_{\mathcal R}(\btheta_o)-\epsilon_r$.
\end{definition}

As with retrospective unlearning, a preemptive defense must preserve the model’s overall utility. 
Achieving robustness by substantially degrading normal performance is unacceptable, particularly because even modest performance gains often require considerable training time and compute.
We therefore seek a released model that both resists post-release attacks and retains strong utility in normal use. 
In the next section, we further explain, both theoretically and empirically, why retrospective unlearning methods fall short in our preemptive setting.

\section{Why Retrospective Unlearning Falls Short}
\label{sec:motivation}


To see why directly repurposing retrospective unlearning may fail to provide the pre-release defense defined in \S\ref{sec:formulation}, we first revisit what retrospective methods actually optimize. Their objectives suppress target behavior at the current parameters, whereas preemptive unlearning is evaluated after an attacker updates those parameters using unseen forbidden-domain data. Retrospective objectives therefore do not account for the gradient induced by post-release fine-tuning on such data. 
\S\ref{sec:theoretical_motivation} formalizes this mismatch, and \S\ref{sec:empirical_motivation} tests its predictions empirically.

\subsection{From Likelihood Suppression to Future Acquisition Gradients}
\label{sec:theoretical_motivation}

We first formalize how retrospective unlearning suppresses the likelihood of undesirable responses, and then show why this suppression alone does not guarantee resistance to future acquisition.

\textbf{Retrospective likelihood suppression.}
Given a prompt $\xbf$ and response $\ybf=(y^1,\ldots,y^{|\ybf|})$, an LLM with parameters $\btheta$ assigns the autoregressive likelihood $\pi_{\btheta}(\ybf\mid\xbf)=\prod_{i=1}^{|\ybf|}\pi_{\btheta}(y^i\mid\xbf,\ybf^{<i})$ and negative log-likelihood (NLL) $\ell_{\btheta}(\xbf,\ybf)=-\log\pi_{\btheta}(\ybf\mid\xbf)$.
Starting from an original model $\btheta_o$ that already contains the target behavior, retrospective unlearning uses a known forget set $\mathcal D_u$ and a retain set $\mathcal D_r$ to produce an unlearned model $\btheta_u$.
It seeks to lower $\pi_{\btheta_u}(\ybf\mid\xbf)$ for $(\xbf,\ybf)\in\mathcal D_u$ while preserving normal behavior on $\mathcal D_r$.
A broad class of methods can be written schematically as
\[
\min_{\btheta}\;
\mathcal L_{\mathrm{RU}}(\btheta)
=
\underbrace{\mathcal L_{\mathrm{sup}}(\btheta;\mathcal D_u)}_{\text{suppress target responses}}
+
\lambda_r
\underbrace{\mathcal L_{\mathrm{ret}}(\btheta;\mathcal D_r,\btheta_o)}_{\text{preserve normal behavior}},
\]
where $\lambda_r\geq0$ controls the forget--retain trade-off. 
Within this common framework, gradient ascent (GA)~\citep{jang2023knowledge,yao2024large} maximizes the forget-set NLL; 
GradDiff~\citep{maini2024tofu} and NPO~\citep{zhang2024negative} add retain- or reference-model constraints;
and RMU~\citep{li2024wmdp} performs representation-level misdirection.
Robustness-oriented work further uses latent adversarial training~\citep{sheshadri2025latent}, smoothness or invariance regularization~\citep{fan2025towards,wang2025invariance}, or optimizer simplification~\citep{lang2026downgrade}. 
Despite their different mechanisms, these methods are centered on suppressing a known target at release; a fuller review appears in \S\ref{sec:appx-related-works}.

\textbf{The missing post-release gradient.}
When a retrospective method is repurposed for pre-release defense, the proxy set $\mathcal D_f$ replaces $\mathcal D_u$, and the optimized model is released as $\btheta_{\mathrm{rel}}$.
The retrospective objective above may suppress forbidden responses at release, but it does not constrain the update induced by a fresh, unseen attacker set $\mathcal D_a$.
This is the central mismatch: likelihood suppression controls the model's current state at $\btheta_{\mathrm{rel}}$, whereas acquisition resistance depends on the directions in which downstream training can move that state.
To isolate this distinction, we theoretically show below that the post-release gradient is the key diagnostic quantity: it explains why retrospective suppression can fail and identifies what a pre-release defense must control beyond current outputs.


Following common malicious fine-tuning settings~\citep{tamirisa2025tamper}, we assume the attacker minimizes the supervised fine-tuning loss $\mathcal L_a(\btheta;\mathcal D_a) \coloneq \mathbb E_{(\xbf,\ybf)\sim\mathcal D_a} [\ell_{\btheta}(\xbf,\ybf)]$.
Our local analysis considers one gradient step $\btheta^{+}=\btheta_{\mathrm{rel}}-\eta\mathbf g_a$, where $\eta>0$ and $\mathbf g_a\coloneq\nabla_{\btheta}\mathcal L_a(\btheta_{\mathrm{rel}};\mathcal D_a)$.
Since forget-related behavior and learning signals can concentrate in localized pathways~\citep{cloud2024gradient}, consider the local orthogonal decomposition $\mathbb R^d=\mathcal S\oplus\mathcal S^\perp$, where $\mathcal S$ contains forbidden-sensitive directions and $\mathcal S^\perp$ is its orthogonal complement.
Let $\Pi_{\mathcal S}$ denote the orthogonal projector onto $\mathcal S$ and $S_{\mathcal F}$ a differentiable surrogate of $\mathsf{Perf}_{\mathcal F}$.
We adopt the idealized local-separation condition
$\Pi_{\mathcal S}\nabla_{\btheta}S_{\mathcal F}(\btheta_{\mathrm{rel}})
=\nabla_{\btheta}S_{\mathcal F}(\btheta_{\mathrm{rel}})$,
so directions in $\mathcal S^\perp$ have no first-order effect on the forbidden score.
The following proposition links one-step acquisition to the component of the attacker gradient in $\mathcal S$.
\begin{restatable}[Future acquisition gain bound; proof deferred to \S\ref{subsec:appx-proof-gradient}]{proposition}{prop}
\label{prop:future_acquisition_bound}
Suppose $S_{\mathcal F}$ is $L_{\mathcal F}$-Lipschitz and $\beta_{\mathcal F}$-smooth in a neighborhood of $\btheta_{\mathrm{rel}}$, with
$\Pi_{\mathcal S}\nabla_{\btheta}S_{\mathcal F}(\btheta_{\mathrm{rel}})
=\nabla_{\btheta}S_{\mathcal F}(\btheta_{\mathrm{rel}})$.
For fixed $\mathcal D_a$, define
$\Delta S_{\mathcal F}\coloneq S_{\mathcal F}(\btheta^{+})-S_{\mathcal F}(\btheta_{\mathrm{rel}})$.
For all sufficiently small $\eta>0$, this one-step acquisition gain satisfies:
\[
\Delta S_{\mathcal F}
=
-\eta
\left\langle
\nabla_{\btheta}S_{\mathcal F}(\btheta_{\mathrm{rel}}),
\mathbf g_a
\right\rangle
+O(\eta^2)
\leq
\eta L_{\mathcal F}
\left\|\Pi_{\mathcal S}\mathbf g_a\right\|_2
+O(\eta^2).
\]
\end{restatable}

\textit{Remark.}
Prop.~\ref{prop:future_acquisition_bound} separates release-time suppression from susceptibility to future acquisition.
Retrospective suppression can make $S_{\mathcal F}(\btheta_{\mathrm{rel}})$ small while leaving $\|\Pi_{\mathcal S}\mathbf g_a\|_2$ large; reducing this norm tightens the first-order upper bound on how rapidly fresh data can increase the forbidden score.
This gap need not reflect a proxy--attacker data mismatch: even when $\mathcal D_a=\mathcal D_f$, likelihood suppression alone does not ensure a small projected gradient.
An example in \S\ref{subsec:appx-counterexample} makes this separation explicit: two released models have identical forbidden and retain outputs, yet their projected gradients and one-step gains differ by controllable factors.
Thus, gradient pathways can remain receptive to future acquisition after output suppression.
\S\ref{sec:empirical_motivation} next investigates these predictions empirically.


\subsection{Release-Time Suppression Leaves Future Acquisition Pathways Open}
\label{sec:empirical_motivation}

Prior studies show that suppressed capabilities can re-emerge after downstream fine-tuning~\citep{hu2025unlearning}.
Echoing these robustness concerns, Prop.~\ref{prop:future_acquisition_bound} highlights that a low release score need not imply a small acquisition gradient along forbidden-sensitive directions.
We examine this gap in knowledge-absent preemptive setting by release failure, pathway prediction, and targeted intervention.


\textbf{Common setup.}
On \tofu~\citep{maini2024tofu}, we fine-tune \texttt{Qwen3.5-2B}~\citep{qwen2026qwen35} on \texttt{retain90} to obtain a retain-only \emph{Reference}.
For each of 20 target authors, we allocate $4/8/8$ facts to $\mathcal D_f$, $\mathcal D_a$, and holdout, respectively, yielding row-disjoint, same-author splits.
Reference is never optimized on these target splits, isolating knowledge-absent preemptive unlearning: the defender uses $\mathcal D_f$ to identify and suppress pathways that disjoint $\mathcal D_a$ may later recruit.
Following \citet{wang2025towards} and \openunlearning~\citep{openunlearning2025}, we measure empirical forbidden capability using extraction strength (\textsf{ES} $\uparrow$), defining $\widehat{\mathsf{Perf}}_{\mathcal F}(\btheta;\mathcal D_a)\coloneq\textsf{ES}(\mathcal D_a)$.
Full protocols appear in \S\ref{subsec:appx-fig2-protocol}.
Fig.~\ref{fig:empirical-gradient-evidence} therefore moves from failure in Fig.~\refs{fig:empirical-gradient-evidence}{a}, to prediction in Fig.~\refs{fig:empirical-gradient-evidence}{b}, and to intervention in Fig.~\refs{fig:empirical-gradient-evidence}{c}.

\begin{figure*}[t]
\centering
\includegraphics[width=\textwidth]{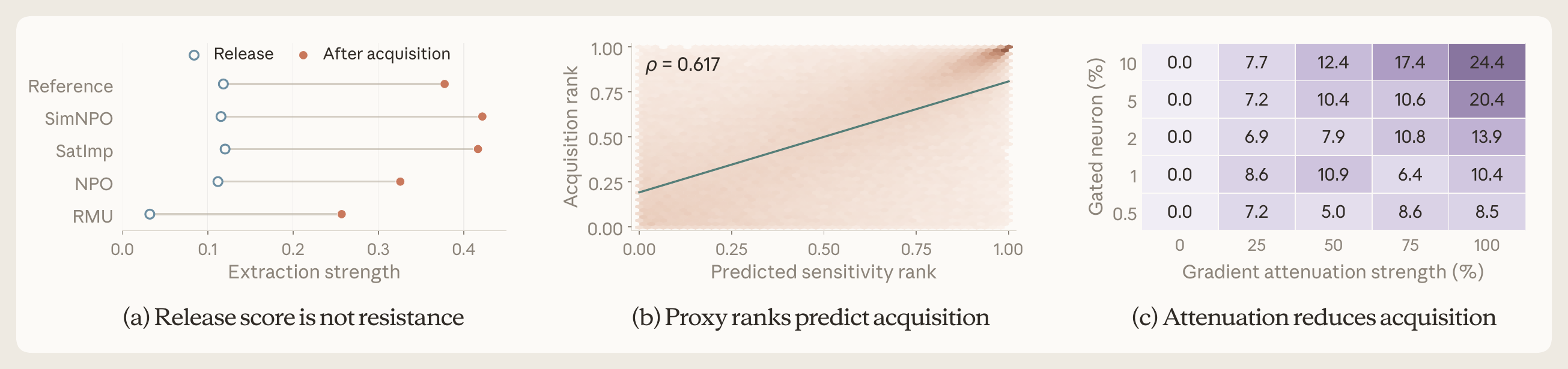}
\caption{\textbf{Empirical motivation for gradient sealing.}
(a) Release and post-acquisition \textsf{ES} for unlearning methods under identical acquisition.
(b) Within-layer channel-rank alignment between $\mathcal D_f$ look-ahead and $\mathcal D_a$ acquisition.
(c) Acquisition-gain reduction across selected-channel fractions and gradient-attenuation strengths.
}
\label{fig:empirical-gradient-evidence}
\end{figure*}


\textbf{Release-time suppression does not confer acquisition resistance.}
Starting from Reference, SimNPO~\citep{fan2025simplicity}, SatImp~\citep{yang2025exploring}, NPO, and RMU produce distinct release states on $\mathcal D_f$ retaining at least $95\%$ of Reference utility.
Fig.~\refs{fig:empirical-gradient-evidence}{a} connects release \textsf{ES} to its value after the same $\mathcal D_a$ acquisition.
Despite widely separated release scores, all method shows substantial acquisition: NPO rises from $0.112$ to $0.326$, and RMU from $0.032$ to $0.257$.
Although their endpoints remain lower, suppression does not prevent subsequent acquisition, consistent with Prop.~\ref{prop:future_acquisition_bound}.
Fig.~\refs{fig:empirical-gradient-evidence}{b} next asks where this residual trainability resides and whether $\mathcal D_f$ can expose it before $\mathcal D_a$ arrives.

\textbf{A proxy look-ahead reveals where future acquisition will flow.}
Let $\ell$ index an MLP layer, $j$ a channel, and $x\in\{f,a\}$ a data split; $\bar A^{s}_{x,\ell j}$ denotes that channel's mean answer-token activation on $\mathcal D_x$ in state $s$.
A short $\mathcal D_f$ look-ahead on a disposable copy defines
$\Delta^f_{\ell j}=|\bar A^{\mathrm{look}}_{f,\ell j}-\bar A^{0}_{f,\ell j}|$.
After resetting to the Reference, independent $\mathcal D_a$ acquisition defines
$\Delta^a_{\ell j}=|\bar A^{\mathrm{acq}}_{a,\ell j}-\bar A^{0}_{a,\ell j}|$.
Here $0$, $\mathrm{look}$, and $\mathrm{acq}$ denote the Reference, look-ahead, and post-acquisition states.
Fig.~\refs{fig:empirical-gradient-evidence}{b} plots their within-layer ranks on the horizontal and vertical axes, respectively, with darker hexagons containing more channels.
The dark upper-right mass thus shows that channels most sensitive to proxy learning also tend to change most during actual acquisition.
The pooled $\rho=0.617$, positive across all 24 layers,  establishes a consistent predictive alignment between proxy sensitivity and subsequent acquisition.
Fig.~\refs{fig:empirical-gradient-evidence}{c} next tests whether this alignment identifies a functional acquisition pathway rather than merely a correlate.

\textbf{Targeted attenuation converts prediction into control.}
Using only the $\Delta^f_{\ell j}$ ranking, we select the top $p\%$ channels per layer and scale their backward/update paths by $m=1-\alpha$ during $\mathcal D_a$ acquisition, while preserving the forward pass and matching per-step update norms to Full-FT.
Fig.~\refs{fig:empirical-gradient-evidence}{c} varies attenuation strength $\alpha$ horizontally and channel coverage $p$ vertically.
Each cell reports
$R(\alpha,p)=100(G_{\mathrm{FT}}-G_{\alpha,p})/G_{\mathrm{FT}}$,
where $G=\textsf{ES}_{\mathrm{after}}-\textsf{ES}_{\mathrm{before}}$ is the acquisition gain.
Moving right strengthens attenuation, moving upward controls more predicted channels, and darker cells indicate less acquisition than Full-FT.
The reduction generally grows toward the upper right, reaching $24.35\%$ at $p=10\%$ and $\alpha=100\%$; under full attenuation, targeted masks also outperform layer- and count-matched random masks at every coverage.
These interventions turn the predictive alignment in Fig.~\refs{fig:empirical-gradient-evidence}{b} into causal evidence that the localized pathways materially support acquisition.

\textbf{Takeaways.}
Taken together, Fig.~\ref{fig:empirical-gradient-evidence} forms a single chain: Fig.~\refs{fig:empirical-gradient-evidence}{a} identifies the failure of static suppression, Fig.~\refs{fig:empirical-gradient-evidence}{b} locates the residual acquisition route in advance, and Fig.~\refs{fig:empirical-gradient-evidence}{c} verifies that the localized route is a functional control point.
A pre-release defense should therefore control where future optimization can flow, not only what the model currently expresses.
Following this logic, GSU uses $\mathcal D_f$ to \emph{expose} acquisition-induced changes, \emph{localize} sensitive gates, and \emph{seal} their gradient pathways.
\S\ref{sec:method} next develops this full \emph{expose--localize--seal} defense.

\section{Gradient-Sealed Unlearning}
\label{sec:method}

Our findings in \S\ref{sec:motivation} motivate sealing acquisition pathways beyond suppressing release-time outputs.
A direct approach penalizes acquisition-gradient norms, but differentiating this penalty introduces second-order derivatives and computational overhead~\citep{zhao2022penalizing}.
GSU avoids this cost by controlling pre-activations, exploiting the zero or near-zero derivatives of ReLU-family activations for sufficiently negative inputs.
Pushing pre-activations into this regime therefore attenuates local gradient factors using only first-order updates.
Because indiscriminate control risks harming normal utility, we first identify acquisition-sensitive gates and restrict sealing to them.
As shown in Fig.~\ref{fig:method}, we \emph{expose} changes induced by proxy learning in \S\ref{subsec:method-expose}, \emph{localize} gates with upward pre-activation shifts in \S\ref{subsec:method-localize}, and \emph{seal} the selected gates jointly with response suppression and retain supervision in \S\ref{subsec:method-seal}.


\begin{figure*}[t]
    \centering
    \includegraphics[width=\textwidth]{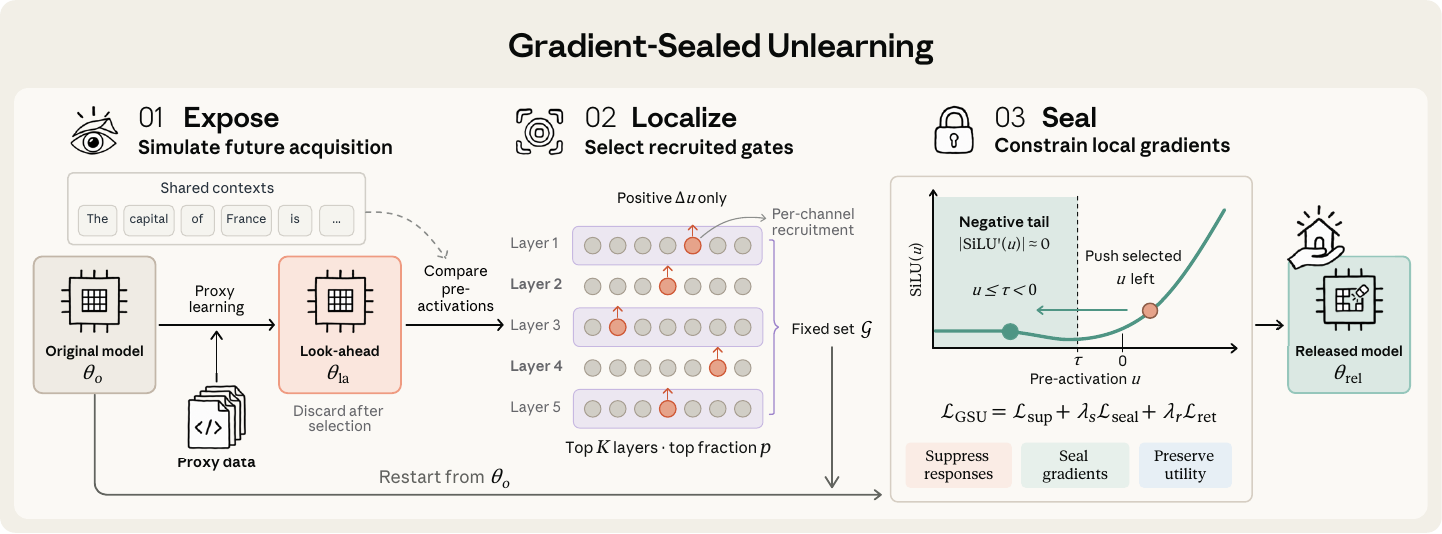}
\caption{\textbf{Overview of GSU.}
\emph{Expose} reveals learning-induced changes by comparing the original model with a disposable proxy-fitted copy under matched contexts.
\emph{Localize} selects a fixed gate set $\mathcal G$ based on upward pre-activation shifts.
\emph{Seal} restarts from $\btheta_o$ and pushes selected pre-activations into the SiLU negative tail to attenuate local gradient factors, alongside response suppression and retain supervision, yielding $\btheta_{\mathrm{rel}}$.
}
    \label{fig:method}
\end{figure*}

\subsection{Expose Acquisition-Induced Changes}
\label{subsec:method-expose}

\textbf{Learning as a diagnostic.}
Under the formulation in \S\ref{sec:formulation}, $\mathcal D_f$ and $\mathcal D_a$ are sampled from the same forbidden-domain distribution $\mathbb P_{\mathcal F}$ but need not overlap.
Since $\mathcal D_a$ is unavailable before release, we use proxy learning on $\mathcal D_f$ to expose gates that unseen acquisition may also recruit.
Starting from $\btheta_o$, we briefly fine-tune a disposable copy using the supervised objective analyzed in \S\ref{sec:theoretical_motivation}:
\begin{equation}
\min\nolimits_{\btheta}\;
\mathbb E_{(\xbf,\ybf)\sim\mathcal D_f}
\bigl[
\ell_{\btheta}(\xbf,\ybf)
\bigr].
\label{eq:gsu-lookahead}
\end{equation}
Let $\btheta_{\mathrm{la}}$ denote the resulting look-ahead parameters.
To determine which gates to target during sealing, we examine how their pre-activations respond to proxy learning.
We therefore compare $\btheta_o$ and $\btheta_{\mathrm{la}}$ on the same proxy examples under matched teacher-forcing contexts: both models receive $(\xbf,\ybf^{<i})$ when predicting $y^i$.
For each such context, we write $u_{\ell j}^{i}(\btheta)\coloneqq u_{\ell j}^{i}(\xbf,\ybf^{<i};\btheta)$ for the gate pre-activation of channel $j$ in MLP layer $\ell$ used to predict $y^i$.
Matching the contexts isolates learning-induced pre-activation changes from differences in generated continuations, providing a consistent signal for selecting acquisition-sensitive gates in the next localization stage.

\subsection{Localize Directionally Recruited Gates}
\label{subsec:method-localize}

\textbf{Directional recruitment.}
To identify which gates to seal, we consider not only how much their pre-activations change, but also whether the changes oppose the intended intervention.
Sealing aims to push pre-activations into sufficiently negative regions, where ReLU and smooth variants such as SiLU~\citep{elfwing2017sigmoid} have zero or near-zero derivatives.
We instantiate this principle with the SiLU gates in the SwiGLU~\citep{shazeer2020glu} MLPs of Llama and Qwen.
Accordingly, we focus on gates whose pre-activations increase under proxy learning, counteracting the downward shift targeted by sealing.
Their directional recruitment score is defined by
\begin{equation}
r_{\ell j}
\coloneqq
\mathbb E_{(\xbf,\ybf)\sim\mathcal D_f}
\biggl[
\frac{1}{|\ybf|}
\sum\nolimits_{i=1}^{|\ybf|}
\operatorname{ReLU}
\bigl(
u_{\ell j}^{i}(\btheta_{\mathrm{la}})
-
u_{\ell j}^{i}(\btheta_o)
\bigr)
\biggr].
\label{eq:unit-score}
\end{equation}
Applying ReLU before averaging prevents decreases at other positions from canceling increases.
The score captures upward shifts even when both pre-activations remain negative.
We next use these scores to select gates for sealing under layer and channel budgets.

\textbf{Hierarchical selection.}
To compare recruitment across layers without favoring larger numerical scales, we normalize each layer's total score by its original pre-activation magnitude.
For a layer $\ell$ of width $d_\ell$, define its magnitude $m_\ell$ and relative recruitment score $R_\ell$ as
\begin{equation}
m_\ell
\coloneqq
\mathbb E_{(\xbf,\ybf)\sim\mathcal D_f}
\biggl[
\frac{1}{|\ybf|}
\sum\nolimits_{i=1}^{|\ybf|}
\sum\nolimits_{j=1}^{d_\ell}
\bigl|u_{\ell j}^{i}(\btheta_o)\bigr|
\biggr],
\qquad
R_\ell
\coloneqq
\frac{1}{m_\ell}
\sum\nolimits_{j=1}^{d_\ell}r_{\ell j}.
\label{eq:gsu-layer-recruitment}
\end{equation}
The $K$ layers with the highest positive $R_\ell$ among those with $m_\ell>0$ form $\mathcal I^\star$.
Within each selected layer, the $\lceil p d_\ell\rceil$ channels with the highest positive $r_{\ell j}$ form $\mathcal N_\ell^\star$, where $p\in(0,1]$ controls the selected fraction.
The resulting layer--channel index set
$\mathcal G\coloneqq
\{(\ell,j):\ell\in\mathcal I^\star,\ j\in\mathcal N_\ell^\star\}$
identifies the gates selected for intervention, enabling us to turn localization into gradient control in the next stage.

\subsection{Seal Selected Gradient Pathways}
\label{subsec:method-seal}
\label{subsec:method-gates}

With $\mathcal G$ fixed, we discard $\btheta_{\mathrm{la}}$ and restart full-parameter optimization from $\btheta_o$.
Proxy learning therefore contributes only the gate selection, not the fitted parameter changes.
We now constrain the selected pre-activations to induce local gradient attenuation in the released model.

\textbf{Sealing local gradients.}
For gate weights $\mathbf w$ and MLP input $\mathbf h$, let $u=\mathbf w^\top\mathbf h$.
By the chain rule, the gradient contribution through this gate is $c\,\phi'(u)\mathbf h$, where $\phi$ is SiLU and $c$ combines the parallel-branch activation with the downstream gradient.
Since $|\phi'(u)|\to0$ as $u\to-\infty$, pushing $u$ sufficiently far into the negative tail makes this local derivative factor small, as illustrated in Fig.~\ref{fig:method}.
This provides a first-order mechanism for gradient control through pre-activations.
We therefore choose a shared threshold $\tau<0$ in this tail and softly penalize violations of $u_{\ell j}^{i}(\btheta)\leq\tau$:
\begin{equation}
\mathcal L_{\mathrm{seal}}(\btheta;\mathcal D_f)
=
\mathbb E_{\substack{
(\xbf,\ybf)\sim\mathcal D_f\\
(\ell,j)\sim\mathcal G
}}
\biggl[
\frac{1}{|\ybf|}
\sum\nolimits_{i=1}^{|\ybf|}
\operatorname{ReLU}^2
\bigl(u_{\ell j}^{i}(\btheta)-\tau\bigr)
\biggr].
\label{eq:seal_loss}
\end{equation}
The shared threshold $\tau$ translates the small-derivative requirement into a common target in the negative tail.
For each selected pre-activation $u$, the scalar penalty has derivative $2\operatorname{ReLU}(u-\tau)$ with respect to $u$, providing a downward signal proportional to the threshold violation.
Once $u\le\tau$, this contribution vanishes, so the seal term exerts no further direct downward pressure.
The loss thus encourages entry into the low-derivative region through a soft constraint rather than hard clipping, allowing sealing to be balanced with utility preservation.

\textbf{Coupling pathways and behavior.}
Sealing targets local gradient factors that output suppression does not explicitly constrain.
To also suppress release-time forbidden responses and preserve normal-domain utility, we combine it with response suppression and retain supervision.
The resulting GSU objective extends the suppression--retention formulation in \S\ref{sec:theoretical_motivation}, yielding
\begin{equation}
\mathcal L_{\mathrm{GSU}}(\btheta;\mathcal D_f,\mathcal D_r)
=
\mathcal L_{\mathrm{sup}}(\btheta;\mathcal D_f)
+
\lambda_s\mathcal L_{\mathrm{seal}}(\btheta;\mathcal D_f)
+
\lambda_r\mathcal L_{\mathrm{ret}}(\btheta;\mathcal D_r),
\label{eq:gsu-objective}
\end{equation}
where $\lambda_s,\lambda_r\geq0$ control the strength of sealing and retention.
We optimize this joint objective with first-order updates to obtain $\btheta_{\mathrm{rel}}$, without differentiating through the look-ahead stage.

GSU thus complements release-time suppression with targeted control of local gradient factors.
When unseen acquisition reuses selected gates whose pre-activations remain in the negative tail, these factors remain small.
\S\ref{sec:experiments} evaluates the resulting acquisition resistance and normal-domain utility.

\section{Experiments}
\label{sec:experiments}

\subsection{Experimental Setup}
\label{sec:exp_setup}

\textbf{Protection scenarios and benchmark construction.}
\textit{1) Robust prevention:} A model provider may wish to prevent downstream fine-tuning from incorporating private information that the released model has never learned.
We simulate this setting with \tofu~\citep{maini2024tofu}, using a retain-only checkpoint trained on the 180 \texttt{retain90} authors.
The remaining 20 \texttt{forget10} authors define the forbidden domain; each contributes ten QA pairs to each of two pools, $B_1$ and $B_2$ (200 QA pairs each).
The attacker uses $B_2$, while the defender uses $B_1$ in the \emph{disjoint} setting and $B_2$ in the \emph{identical} setting.
\textit{2) Robust removal:} A model may already contain hazardous knowledge that must be removed and kept from being restored through downstream fine-tuning.
We study this setting on \texttt{WMDP-Bio} and \texttt{WMDP-Cyber}~\citep{li2024wmdp}, starting from a knowledgeable checkpoint.
Each domain's forget corpus is split into source-document pools $B_1$ for defense and $B_2$ for attack, with the full corresponding retain corpus supporting defense.
We report only the disjoint setting: since the target knowledge is already present, identical defense and attack data would degrade to conventional robust unlearning.

\textbf{Models and baselines.}
On \tofu, we evaluate instruction-tuned \texttt{Llama3-1B/3B/8B}~\citep{dubey2024llama,meta2024llama32} and \texttt{Qwen3.5-2B/4B/9B}~\citep{qwen2026qwen35}; on \wmdp, we use \texttt{Zephyr-7B-$\beta$}~\citep{tunstall2023zephyr}.
\emph{No Defense} denotes the shared pre-defense checkpoint: the retain-only base on \tofu and the original checkpoint on \wmdp.
We compare 5 conventional unlearning methods---GradDiff~\citep{maini2024tofu}, NPO~\citep{zhang2024negative}, RMU~\citep{li2024wmdp}, WGA~\citep{wang2025rethinking}, and SatImp~\citep{yang2025exploring}---and 3 robust defenses, RepNoise~\citep{rosati2024representation}, ILU~\citep{wang2025invariance}, and NPO+SAM~\citep{fan2025towards}, all with retention.

\textbf{Defense and attack settings.}
Defense uses 10 epochs with a base learning rate of $10^{-5}$ on \tofu, and 80 updates at $4\times10^{-6}$ on \wmdp.
All defense and attack runs use an effective batch size of 16.
Following fine-tuning-based evaluations of tamper resistance and robust unlearning~\citep{tamirisa2025tamper,fan2025towards}, we attack released models through supervised full-parameter fine-tuning on $\mathcal D_a$.
\tofu attacks run for 10 epochs, using a learning rate calibrated once per model on \emph{No Defense} and shared across methods; \wmdp attacks run for 150 updates at $4\times10^{-6}$.

\textbf{Evaluation metrics.}
We report early-mean and fifth-checkpoint forbidden scores:
$\overline{\textsf{ES}}_3$ and $\textsf{ES}_5$ on \tofu use extraction strength~\citep{wang2025towards,openunlearning2025} at attack epochs 1--3 and 5;
$\overline{\textsf{F}}_3$ and $\textsf{F}_5$ on \wmdp use domain accuracy at steps 25/50/75 and 125.
$\textsf{UA}_{90}$ is the maximum evaluated forbidden score with utility $\geq0.9U_{\mathrm{ref}}$.
For fair comparison, UWC~\citep{wang2025towards} calibrates all main-comparison releases to utility $\geq0.95U_{\mathrm{ref}}$, so release utility is omitted from the tables.
Scores are percentages (lower is better); bold and underline mark the best and second-best defended methods using unrounded values.
\S\ref{sec:appx-exp} provides full definitions and configurations.

\subsection{Main Comparisons}
\label{sec:main_results}

\begin{table*}[t]
\centering
\caption{\textbf{\tofu results of Llama3 and Qwen3.5 families under disjoint and identical prevention.}}
\label{table:main_tofu}
\small
\setlength{\tabcolsep}{2.5pt}
\renewcommand{\arraystretch}{1.2}
\resizebox{\textwidth}{!}{%
\begin{tabular}{cl*{18}{c}}
\toprule
& &
\multicolumn{3}{c}{\texttt{Llama-3.2-1B}} &
\multicolumn{3}{c}{\texttt{Llama-3.2-3B}} &
\multicolumn{3}{c}{\texttt{Llama-3.1-8B}} &
\multicolumn{3}{c}{\texttt{Qwen3.5-2B}} &
\multicolumn{3}{c}{\texttt{Qwen3.5-4B}} &
\multicolumn{3}{c}{\texttt{Qwen3.5-9B}} \\
\cmidrule(lr){3-5}
\cmidrule(lr){6-8}
\cmidrule(lr){9-11}
\cmidrule(lr){12-14}
\cmidrule(lr){15-17}
\cmidrule(lr){18-20}
& \textbf{Method}
& $\overline{\textsf{ES}}_3\downarrow$
& $\textsf{ES}_5\downarrow$
& $\textsf{UA}_{90}\downarrow$
& $\overline{\textsf{ES}}_3\downarrow$
& $\textsf{ES}_5\downarrow$
& $\textsf{UA}_{90}\downarrow$
& $\overline{\textsf{ES}}_3\downarrow$
& $\textsf{ES}_5\downarrow$
& $\textsf{UA}_{90}\downarrow$
& $\overline{\textsf{ES}}_3\downarrow$
& $\textsf{ES}_5\downarrow$
& $\textsf{UA}_{90}\downarrow$
& $\overline{\textsf{ES}}_3\downarrow$
& $\textsf{ES}_5\downarrow$
& $\textsf{UA}_{90}\downarrow$
& $\overline{\textsf{ES}}_3\downarrow$
& $\textsf{ES}_5\downarrow$
& $\textsf{UA}_{90}\downarrow$ \\
\midrule
& No Defense
& 9.04 & 16.33 & 50.08
& 9.36 & 19.41 & 50.75
& 9.51 & 19.63 & 53.68
& 6.48 & 6.69 & 6.84
& 9.55 & 26.36 & 66.12
& 19.44 & 90.92 & 98.53 \\
\cmidrule(r){2-20}

\multirow{9}{*}{\rotatebox[origin=c]{90}{Disjoint Prevention}}
& GradDiff
& 8.62 & 15.12 & 51.41
& 9.10 & 20.62 & 35.48
& 8.69 & \underline{18.83} & \textbf{44.56}
& 5.78 & 6.25 & 6.96
& 9.45 & 26.87 & 80.24
& 19.88 & 93.40 & 98.72 \\
& NPO
& 8.04 & 13.57 & 47.06
& 8.34 & 19.99 & 4.42
& \underline{7.84} & 20.42 & 95.80
& 4.85 & 5.07 & 5.74
& 9.06 & 25.57 & 64.99
& 19.72 & 93.15 & 98.86 \\
& RMU
& 9.02 & 16.17 & 49.59
& 9.42 & \textbf{19.37} & 50.40
& 9.35 & 19.77 & 86.55
& 6.47 & 6.71 & 6.88
& 9.53 & 26.48 & 68.21
& 19.75 & \underline{92.90} & \underline{98.32} \\
& WGA
& 8.31 & 14.34 & 48.19
& 8.61 & 20.57 & 48.01
& 8.40 & 21.93 & 89.77
& \underline{4.55} & 5.46 & 6.07
& 9.27 & 25.28 & 69.94
& 19.79 & 93.10 & 99.16 \\
& SatImp
& 9.02 & 17.22 & 55.32
& 9.48 & 23.06 & 55.90
& 8.52 & 24.66 & 91.11
& 5.93 & 6.07 & 6.35
& 9.26 & 26.94 & 94.69
& 20.85 & 93.05 & 99.36 \\
& RepNoise
& 8.67 & 16.69 & 53.16
& 9.05 & 23.08 & 40.49
& 8.80 & 21.60 & 86.93
& 5.54 & 6.12 & 6.84
& 9.48 & 24.76 & 63.72
& 19.71 & 92.95 & 98.53 \\
& ILU
& \underline{7.14} & \underline{12.93} & \underline{30.52}
& \textbf{8.15} & 20.42 & 4.47
& 8.42 & 21.60 & 87.12
& 4.93 & 5.40 & 5.95
& \textbf{8.24} & 25.06 & \underline{61.86}
& 19.80 & 93.30 & 98.80 \\
& NPO+SAM
& 7.45 & 13.73 & 45.63
& 8.35 & 21.39 & \underline{4.38}
& 8.36 & 21.33 & 86.74
& 4.73 & \underline{5.07} & \underline{5.34}
& 8.47 & \underline{22.80} & 62.22
& \underline{19.67} & 93.03 & \textbf{98.10} \\
& \textbf{GSU (Ours)}
& \textbf{6.60} & \textbf{12.40} & \textbf{27.87}
& \underline{8.21} & \underline{19.89} & \textbf{3.77}
& \textbf{7.78} & \textbf{18.03} & \underline{86.17}
& \textbf{4.06} & \textbf{4.44} & \textbf{5.03}
& \underline{8.28} & \textbf{21.81} & \textbf{59.40}
& \textbf{19.63} & \textbf{92.85} & 98.39 \\
\midrule

\multirow{9}{*}{\rotatebox[origin=c]{90}{Identical Prevention}}
& GradDiff
& 7.42 & 11.62 & 49.31
& 7.94 & 13.57 & 13.57
& 8.76 & 16.46 & 47.82
& 4.75 & 5.36 & 5.68
& 8.26 & 17.80 & 69.17
& 17.50 & 84.07 & 98.60 \\
& NPO
& 5.51 & 8.51 & 20.60
& 5.94 & 9.71 & \underline{2.81}
& \textbf{5.73} & 8.93 & \underline{2.74}
& 3.38 & 3.62 & 4.15
& 6.30 & 9.35 & 38.22
& 11.31 & 66.67 & 98.22 \\
& RMU
& 9.03 & 15.85 & 48.43
& 9.25 & 19.44 & 50.75
& 9.39 & 19.70 & 54.91
& 6.50 & 6.65 & 6.84
& 9.60 & 26.67 & 67.12
& 19.26 & 86.95 & 98.68 \\
& WGA
& 6.85 & 11.51 & 49.86
& 7.22 & 13.76 & 21.64
& 7.41 & 13.12 & 68.37
& \textbf{1.13} & \textbf{1.67} & \textbf{2.44}
& 7.82 & 15.18 & 81.74
& 16.58 & 83.30 & 98.75 \\
& SatImp
& 8.86 & 17.08 & 54.27
& 9.23 & 23.88 & 56.27
& 8.65 & 22.53 & 89.14
& 5.86 & 6.15 & 6.36
& 9.32 & 27.46 & 90.31
& 20.14 & 85.18 & 98.66 \\
& RepNoise
& 8.37 & 14.89 & 44.87
& 8.44 & 18.20 & 51.19
& 8.78 & 17.92 & 76.43
& 5.48 & 5.85 & 6.24
& 9.39 & 23.58 & 78.31
& 19.09 & 86.67 & 98.53 \\
& ILU
& 5.90 & 9.30 & 20.54
& 6.24 & 11.48 & 3.19
& 5.86 & 8.92 & 5.66
& 3.37 & 3.61 & 4.39
& \textbf{5.80} & 9.83 & \textbf{3.21}
& 11.23 & 66.56 & \underline{95.60} \\
& NPO+SAM
& \underline{5.28} & \textbf{8.17} & \textbf{17.56}
& \textbf{5.48} & \textbf{8.61} & 2.90
& 5.79 & \textbf{7.28} & 3.04
& 3.04 & 3.60 & 4.04
& 6.21 & \textbf{8.09} & 30.75
& \underline{9.43} & \underline{48.08} & \textbf{95.20} \\
& \textbf{GSU (Ours)}
& \textbf{5.24} & \underline{8.34} & \underline{20.52}
& \underline{5.91} & \underline{9.39} & \textbf{2.49}
& \underline{5.75} & \underline{8.88} & \textbf{2.64}
& \underline{2.73} & \underline{3.13} & \underline{3.95}
& \underline{6.17} & \underline{9.31} & \underline{30.70}
& \textbf{9.04} & \textbf{45.08} & 96.11 \\
\bottomrule
\end{tabular}%
}
\end{table*}

\textbf{Results on \tofu.}
GSU ranks first or second in 34 of 36 model--setting--metric comparisons in Tab.~\ref{table:main_tofu}.
Under disjoint prevention, it achieves the lowest $\textsf{ES}_5$ on five of six models and the lowest $\textsf{UA}_{90}$ on four, indicating resistance both at a fixed attack checkpoint and among utility-preserving attacks.
For example, on \texttt{Llama-3.2-1B}, GSU reduces $\textsf{UA}_{90}$ from 50.08 without defense to 27.87, compared with 30.52 for the strongest baseline.
Under identical prevention, GSU remains first or second in 17 of 18 comparisons.
On \texttt{Qwen3.5-9B}, it achieves an $\textsf{ES}_5$ of 45.08, versus 48.08 for NPO+SAM and 90.92 without defense.
These results demonstrate broad improvements across models and both prevention settings.

\begin{wraptable}[12]{r}{0.42\textwidth}
\vspace{\dimexpr-\intextsep-0.7pt\relax}
\centering

\captionsetup{
    font=footnotesize,
    position=top,
    skip=2pt,
    belowskip=0pt
}
\caption{\textbf{\wmdp{} results under disjoint attacks.}}
\label{table:main_wmdp}

\scriptsize
\setlength{\tabcolsep}{1.5pt}
\renewcommand{\arraystretch}{1.03}
\begin{tabular*}{\linewidth}{@{\extracolsep{\fill}}l*{6}{c}@{}}
\toprule
& \multicolumn{3}{c}{\texttt{Bio}}
& \multicolumn{3}{c}{\texttt{Cyber}} \\
\cmidrule(lr){2-4}\cmidrule(lr){5-7}
\textbf{Method}
& $\overline{\textsf{F}}_3\downarrow$
& $\textsf{F}_5\downarrow$
& $\textsf{UA}_{90}\downarrow$
& $\overline{\textsf{F}}_3\downarrow$
& $\textsf{F}_5\downarrow$
& $\textsf{UA}_{90}\downarrow$ \\
\midrule
No Defense
& 65.23 & 65.36 & 65.99
& 43.78 & 43.83 & 44.14 \\
\midrule
GradDiff
& 37.54 & 37.72 & 37.87
& 33.21 & 33.71 & 35.51 \\
NPO
& 35.44 & 36.39 & 36.48
& 32.64 & \underline{31.08} & 33.19 \\
RMU
& \textbf{34.04} & \underline{34.83} & \textbf{34.24}
& 35.09 & 32.48 & 32.57 \\
WGA
& 36.89 & 37.01 & 37.18
& \underline{30.99} & 31.72 & \textbf{30.11} \\
SatImp
& 39.03 & 39.19 & 39.22
& 33.93 & 34.31 & 33.84 \\
RepNoise
& 36.18 & 35.69 & 35.71
& 34.42 & 34.99 & 34.37 \\
ILU
& 39.81 & 39.98 & 39.93
& 31.98 & 33.09 & 31.87 \\
NPO+SAM
& 38.28 & 38.47 & 38.53
& 35.51 & 35.56 & 34.99 \\
\textbf{GSU (Ours)}
& \underline{34.67} & \textbf{34.13} & \underline{34.98}
& \textbf{30.16} & \textbf{30.23} & \underline{30.92} \\
\bottomrule
\end{tabular*}
\vspace{-4pt}
\end{wraptable}

\textbf{Results on \wmdp.}
GSU ranks first or second across both domains and all three metrics in Tab.~\ref{table:main_wmdp}.
It achieves the lowest $\textsf{F}_5$ on both Bio and Cyber, reducing accuracy to 34.13 and 30.23, respectively, compared with 65.36 and 43.83 without defense.
These scores also improve over the strongest competing results of 34.83 on Bio and 31.08 on Cyber.
GSU further attains the lowest $\overline{\textsf{F}}_3$ on Cyber and the second-lowest $\textsf{UA}_{90}$ in both domains.
Together, these results extend the acquisition-resistance gains observed on \tofu{} to hazardous-knowledge restoration under disjoint attacks.

\subsection{Further Analyses}
\label{sec:ablation}

We connect acquisition resistance, selected-gate responses, and benign learnability.
Under disjoint \tofu, the first two analyses use \texttt{Qwen3.5-2B},
while benign fine-tuning covers three models; \S\ref{sec:appx-add-result} details the frozen settings.

\begin{figure}[t]
    \centering
    \includegraphics[width=\textwidth]{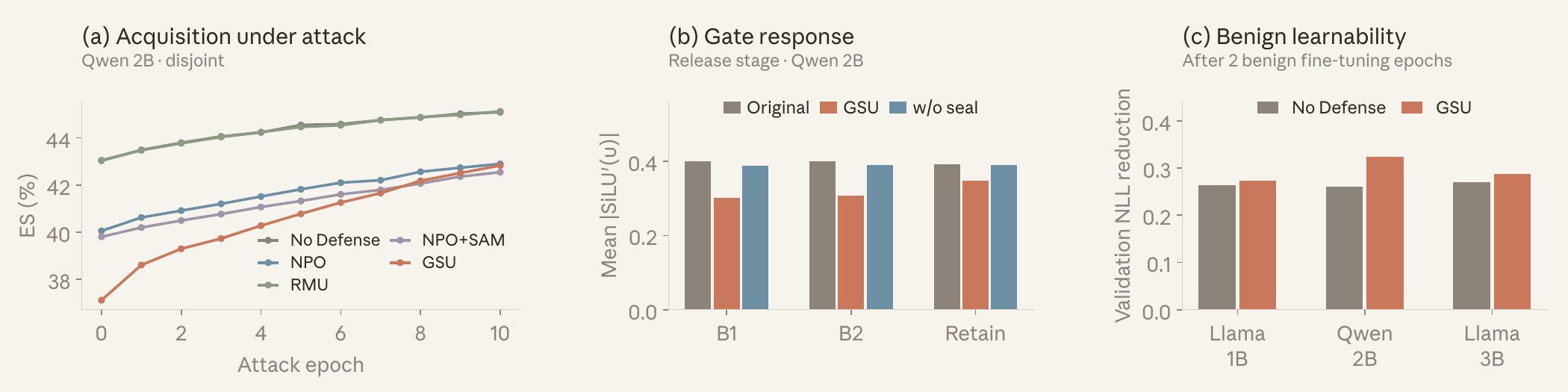}
    \captionsetup{aboveskip=7pt}
    \caption{\textbf{Acquisition resistance, gate response, and benign learnability.}
    (a) \textsf{ES} over ten attack epochs.
    (b) Mean absolute SiLU derivative over selected gates at release on $B_1$, $B_2$, and retain inputs.
    (c) Validation NLL reduction after two epochs of benign fine-tuning, interpreted relative to initial loss.
    Panels (a,b) use \texttt{Qwen3.5-2B} under disjoint \tofu; all results are single runs.}
    \label{fig:gsu-further-analyses}
\end{figure}

\textbf{Acquisition under attack.}
In Fig.~\refs{fig:gsu-further-analyses}{a}, GSU has the lowest \textsf{ES}
among the five methods at all eleven observed checkpoints.
At epoch~5, it reaches 4.713\%, versus 5.056\% for NPO and 5.102\% for NPO+SAM;
at epoch~10, GSU and NPO+SAM reach 5.377\% and 5.533\%, respectively.
This supports lower extraction throughout the measured attack budget.
The gap narrows and GSU starts lower, so the comparison does not establish uniformly slower learning.

\textbf{Release-stage gate response.}
Fig.~\refs{fig:gsu-further-analyses}{b} probes selected-gate responses at release.
GSU reduces mean absolute SiLU derivatives relative to the original control by
approximately 24\% on $B_1$ and 23\% on $B_2$, versus 11\% on retain inputs;
without sealing stays closer to the control.
This is consistent with targeted local attenuation; matched post-attack probes
in \S\ref{subsec:appx-gsu-postattack-gates} examine its partial reversal and residual differences.

\textbf{Benign learnability.}
Fig.~\refs{fig:gsu-further-analyses}{c} tests whether benign learning remains possible.
After two economics fine-tuning epochs, GSU on \texttt{Llama-3.2-1B/3B}
and \texttt{Qwen3.5-2B} reduces validation NLL to within 0.005 of the No Defense controls.
This supports retained benign optimization ability; larger reductions also reflect
higher initial losses, not improved learning efficiency.

\textbf{Additional results.}
\S\ref{subsec:appx-gsu-attack-utility} reports attack-time utility.
\S\ref{subsec:appx-gsu-complete-ablation} and \S\ref{subsec:appx-gsu-seal-ablation}
provide component ablations and sealing trajectories.
\S\ref{subsec:appx-gsu-gates-dynamics} and \S\ref{subsec:appx-gsu-postattack-gates}
examine gate responses at release and after attack.
\S\ref{subsec:appx-gsu-benign} reports benign accuracy and NLL trajectories,
while \S\ref{subsec:appx-gsu-benign-attack} tests protection after benign fine-tuning.

\section{Conclusion}
\label{sec:conclusion}

Preemptive unlearning requires controlling not only released behavior, but also what downstream optimization can acquire.
GSU addresses this gap by exposing proxy-induced changes, localizing recruited gates, and sealing their local gradient pathways.
Restarting from the pre-defense weights, it pushes selected pre-activations toward a shared SiLU negative-tail threshold alongside response suppression and retain supervision.
The framework distinguishes prevention and reacquisition resistance from benign adaptability.
Across the evaluated settings, GSU limits forbidden acquisition while retaining benign fine-tunability, with component ablations and matched gate probes supporting targeted sealing beyond output suppression alone.

\bibliography{ref}
\bibliographystyle{iclr2027_conference}
\clearpage


\appendix

\vspace*{1mm}
\begin{center}
    \LARGE \bf {Appendix of \textit{Gradient-Sealed Unlearning}}
\end{center}

\definecolor{kleinblue}{rgb}{0,0.18,0.65}

\etocdepthtag.toc{mtappendix}
\etocsettagdepth{mtchapter}{none}
\etocsettagdepth{mtappendix}{subsection}
{
  \hypersetup{linkcolor=kleinblue}
  \tableofcontents
}

\clearpage

\section*{Overview of the Appendix}

This appendix supplies notation, related work, proofs, pseudocode, experimental settings, and supplementary analyses supporting the main paper.

\begin{itemize}[leftmargin=*,itemsep=0.3em,parsep=0em,partopsep=0em,before=\vspace{-0.4em},after=\vspace{0em}]
    \item \S\ref{sec:appx-notations} summarizes the notation.
    \item \S\ref{sec:appx-related-works} reviews the four closest research threads.
    \item \S\ref{sec:appx-theory} proves Prop.~\ref{prop:future_acquisition_bound} and gives a minimal counterexample.
    \item \S\ref{sec:appx-pseudocode} presents GSU pseudocode.
    \item \S\ref{sec:appx-exp} details data splits, metrics, implementations, and optimization.
    \item \S\ref{sec:appx-add-result} connects acquisition and utility to component ablations and gate responses, then examines benign learning and protection after adaptation.
    \item \S\ref{sec:appx-limitations} discusses scope and limitations.
\end{itemize}

\section{Notations}
\label{sec:appx-notations}

This section summarizes the main notations in Tab.~\ref{tab:notation}.

{
\renewcommand{\arraystretch}{1.05}
\begin{table}[H]
\centering
\caption{Core notation used in the paper.}
\label{tab:notation}
\small
\begin{tabular}{@{}>{\raggedright\arraybackslash}p{0.28\textwidth}@{\hspace{10pt}}>{\raggedright\arraybackslash}p{0.66\textwidth}@{}}
\toprule
\textbf{Notation} & \textbf{Description} \\
\midrule
$\xbf,\ybf$
& Prompt and autoregressive response. \\

$\pi_{\btheta},\ell_{\btheta}$
& Response likelihood and response NLL. \\

$\mathcal F,\mathcal R$; $\mathbb P_{\mathcal F},\mathbb P_{\mathcal R}$
& Forbidden and normal domains, and their data distributions. \\

$\mathcal D_f,\mathcal D_a,\mathcal D_r$
& Defender proxy, unseen attacker, and retain sets. \\

$\btheta_o,\btheta_{\mathrm{la}},\btheta_{\mathrm{rel}},\btheta_{\mathrm{atk}}$
& Original, look-ahead, released, and post-attack parameters. \\

$\mathcal U,\mathcal A,\mathfrak A$
& Pre-release defense, acquisition attack, and attack class. \\

$\mathsf{Perf}_{\mathcal F},\mathsf{Perf}_{\mathcal R}$
& Forbidden capability and normal utility; hats denote empirical estimates. \\

$\epsilon_f,\epsilon_u,\epsilon_r$
& Forbidden-capability ceiling, usable-utility floor, and allowed release-time utility drop. \\

$\mathcal U_T,\textsf{UA}_T$
& Usable attack checkpoints through $T$ and their maximum empirical forbidden capability. \\

$\mathcal L_a,\mathbf g_a$
& Attacker loss and its gradient at $\btheta_{\mathrm{rel}}$. \\

$S_{\mathcal F},\Delta S_{\mathcal F}$
& Differentiable forbidden-domain score and its one-step acquisition gain. \\

$\mathcal S,\Pi_{\mathcal S}$
& Forbidden-sensitive parameter subspace and its orthogonal projector. \\

$u_{\ell j}^{i}(\btheta)$
& Gate pre-activation at layer $\ell$, channel $j$, under context $(\xbf,\ybf^{<i})$ used to predict $y^i$. \\

$\phi,\phi'$
& SiLU activation and its derivative. \\

$d_\ell,K_{\mathrm{la}},K,p$
& Layer width, look-ahead steps, selected layer count, and within-layer channel fraction. \\

$r_{\ell j},m_\ell,R_\ell$
& Directional channel score, original layer magnitude, and normalized layer score. \\

$\mathcal I^\star,\mathcal N_\ell^\star,\mathcal G$
& Selected layers, selected channels within layer $\ell$, and the fixed layer--channel gate set. \\

$\tau<0$
& Shared pre-activation threshold in the SiLU negative tail. \\

$\mathcal L_{\mathrm{sup}},\mathcal L_{\mathrm{ret}}$
& Response suppression and normal-domain retention losses. \\

$\mathcal L_{\mathrm{seal}},\mathcal L_{\mathrm{GSU}}$
& Selected-gate seal penalty and joint defense objective. \\

$\lambda_s,\lambda_r$
& Sealing and retention weights. \\
\bottomrule
\end{tabular}
\end{table}
}

\section{Detailed Related Work}
\label{sec:appx-related-works}

This section separates four neighboring research threads.
\S\ref{subsec:appx-related-retrospective} reviews retrospective LLM unlearning and its robustness, \S\ref{subsec:appx-related-openweight} covers open-weight and tamper-resistant safeguards, \S\ref{subsec:appx-related-data} discusses data-side preemptive protection, and \S\ref{subsec:appx-related-representation} summarizes representation localization and pathway control.
Together, these threads motivate resisting unseen future learning in released weights.

\subsection{Retrospective LLM Unlearning and Robustness}
\label{subsec:appx-related-retrospective}
Retrospective LLM unlearning extends the broader goal of removing a designated training influence from an already trained model~\citep{bourtoule2021machine} to knowledge and behaviors encoded by LLMs.
Most methods optimize a release-time forget--retain trade-off.
At the output level, GA maximizes forget-set loss~\citep{jang2023knowledge,yao2024large}, GradDiff couples forgetting with retain training~\citep{maini2024tofu}, and NPO uses a reference-relative preference objective~\citep{zhang2024negative}.
Subsequent variants simplify or rebalance these losses, including SimNPO~\citep{fan2025simplicity}, WGA~\citep{wang2025rethinking}, SatImp~\citep{yang2025exploring}, forget-only loss adjustment~\citep{wang2025llm}, and reinforcement-based unlearning~\citep{zhang2025rule}.
A complementary line intervenes inside the network: RMU misdirects target representations~\citep{li2024wmdp}, activation redirection changes internal states~\citep{shen2026llm}, PaCE reduces the projection of persona vectors~\citep{li2026llmpersona}, Ssiuu regularizes spurious unlearning neurons~\citep{yang2026erase}, and reasoning- or belief-guided objectives target model-generated alternatives rather than only labeled answers~\citep{liao2026explainable,li2026llm}.

Benchmarks such as \tofu~\citep{maini2024tofu}, \wmdp~\citep{li2024wmdp}, and \muse~\citep{shi2025muse} evaluate privacy, copyright, and hazardous-knowledge removal, while overlap-aware settings such as BLUR test whether forgetting remains meaningful when forget and retain distributions intersect~\citep{hu2025blur}.
However, a low target score immediately after unlearning can reflect suppression or obfuscation rather than durable removal.
Prompting, probing, relearning, and benign downstream fine-tuning can recover apparently forgotten behavior~\citep{lynch2024eight,hu2025unlearning}, motivating robustness-enhancing approaches based on latent adversarial training~\citep{sheshadri2025latent}, sharpness-aware optimization~\citep{fan2025towards}, downstream invariance~\citep{wang2025invariance}, and optimizer simplification~\citep{lang2026downgrade}.
Our \wmdp regime directly evaluates this removal-plus-resistance problem.
Preemptive unlearning is broader in one crucial respect: it also covers a target capability absent from the original checkpoint and asks whether unseen future data can acquire it after release, rather than only whether previously encoded behavior can be recovered.

\subsection{Open-Weight Safety and Tamper-Resistant Safeguards}
\label{subsec:appx-related-openweight}
Weight access allows downstream users to bypass inference-time moderation by directly changing the model.
Aligned LLMs can lose safety after malicious or even benign fine-tuning~\citep{qi2024finetuning,lermen2023lora}, and safeguarded outputs may themselves provide supervision for eliciting harmful capabilities~\citep{kaunismaa2026eliciting}.
Input--output guard models such as Llama Guard and WildGuard~\citep{inan2023llamaguard,han2024wildguard} remain useful at deployment but cannot prevent weight-level adaptation.
Model-side defenses include representation noising~\citep{rosati2024representation}, tamper-resistant safeguards~\citep{tamirisa2025tamper}, perturbation-aware alignment~\citep{huang2024vaccine,huang2025booster}, procedure-constrained defenses~\citep{hsu2024safe,chen2025sdd}, and filtered pretraining~\citep{obrien2025deep}.
These works mainly preserve broad safety alignment or refusal behavior.
Preemptive unlearning instead asks whether clean, disjoint future data can acquire a designated capability while the released model retains normal utility.

\subsection{Data-Side Preemptive Protection}
\label{subsec:appx-related-data}
Unlearnable examples perturb training data so that a learner cannot easily acquire their semantics~\citep{huang2021unlearnable,fowl2021adversarial,fu2022robust}.
Related privacy and copyright defenses include Fawkes for face recognition~\citep{shan2020fawkes}, Anti-DreamBooth for personalized generation~\citep{vanle2023antidreambooth}, and artwork protections such as Glaze and Nightshade~\citep{liang2023adversarial,shan2023glaze,shan2024nightshade}.
Adaptive evaluations show that perturbation-based protection can be brittle under robust preprocessing or retraining~\citep{honig2025adversarial}.
The key distinction is control: data-side methods assume that the protected examples are those later used for training, whereas our defender cannot modify or even observe the attacker's future data.
GSU therefore encodes resistance into the released weights using a separate proxy set.

\subsection{Representation Localization and Pathway Control}
\label{subsec:appx-related-representation}
Model editing localizes and modifies internal computations associated with facts~\citep{meng2022locating,meng2023massediting}; concept erasure removes linearly represented information~\citep{ravfogel2020null,belrose2023leace}; and representation engineering identifies low-dimensional directions associated with high-level behavior~\citep{zou2023representation,arditi2024refusal}.
Gradient routing further shows that data-dependent learning signals can be localized to computational pathways~\citep{cloud2024gradient}.
GSU differs in objective: it does not merely change a current output or erase a linearly decodable feature, but targets the gate derivatives through which future domain data would update the model.

\section{Proofs and Supporting Analysis}
\label{sec:appx-theory}

This section provides the proof and supporting analysis for Prop.~\ref{prop:future_acquisition_bound}.
\S\ref{subsec:appx-proof-gradient} first derives the future-acquisition bound with an explicit second-order remainder; \S\ref{subsec:appx-counterexample} then gives a minimal ReLU counterexample showing how identical release-time scores can conceal different acquisition receptivity and how activation suppression reduces it.

\subsection{Proof of Proposition~\ref{prop:future_acquisition_bound}}
\label{subsec:appx-proof-gradient}

\prop*

\begin{proof}
Fix $\mathcal D_a$ and abbreviate
\[
\btheta_0 \coloneq \btheta_{\mathrm{rel}},
\qquad
\mathbf g \coloneq \mathbf g_a.
\]
Since $\mathbf g$ is evaluated at $\btheta_0$ for the fixed dataset $\mathcal D_a$, it is independent of the step size $\eta$.
Let $\mathcal N$ be a neighborhood of $\btheta_0$ on which $S_{\mathcal F}$ is both $L_{\mathcal F}$-Lipschitz and $\beta_{\mathcal F}$-smooth.
Because $\mathbf g$ is fixed, there exists $\eta_0>0$ such that, for every $\eta\in(0,\eta_0]$, the entire line segment
\[
\left\{
\btheta_0-t\eta\mathbf g : t\in[0,1]
\right\}
\]
is contained in $\mathcal N$.

Applying the fundamental theorem of calculus along this segment gives
\begin{align*}
\Delta S_{\mathcal F}
&=
S_{\mathcal F}(\btheta_0-\eta\mathbf g)
-
S_{\mathcal F}(\btheta_0) \\
&=
-\eta
\int_0^1
\left\langle
\nabla_{\btheta}S_{\mathcal F}
(\btheta_0-t\eta\mathbf g),
\mathbf g
\right\rangle
\,\mathrm dt \\
&=
-\eta
\left\langle
\nabla_{\btheta}S_{\mathcal F}(\btheta_0),
\mathbf g
\right\rangle
+
R_\eta,
\end{align*}
where
\[
R_\eta
\coloneq
-\eta
\int_0^1
\left\langle
\nabla_{\btheta}S_{\mathcal F}
(\btheta_0-t\eta\mathbf g)
-
\nabla_{\btheta}S_{\mathcal F}(\btheta_0),
\mathbf g
\right\rangle
\,\mathrm dt.
\]
By the Cauchy--Schwarz inequality and the $\beta_{\mathcal F}$-smoothness of $S_{\mathcal F}$,
\begin{align*}
|R_\eta|
&\leq
\eta
\int_0^1
\left\|
\nabla_{\btheta}S_{\mathcal F}
(\btheta_0-t\eta\mathbf g)
-
\nabla_{\btheta}S_{\mathcal F}(\btheta_0)
\right\|_2
\|\mathbf g\|_2
\,\mathrm dt \\
&\leq
\eta
\int_0^1
\beta_{\mathcal F}
\left\|
t\eta\mathbf g
\right\|_2
\|\mathbf g\|_2
\,\mathrm dt \\
&=
\frac{\beta_{\mathcal F}}{2}
\eta^2
\|\mathbf g\|_2^2.
\end{align*}
Since $\mathbf g=\mathbf g_a$ is fixed with respect to $\eta$, the preceding estimate implies that $R_\eta=O(\eta^2)$. 
Therefore,
\[
\Delta S_{\mathcal F}
=
-\eta
\left\langle
\nabla_{\btheta}S_{\mathcal F}(\btheta_{\mathrm{rel}}),
\mathbf g_a
\right\rangle
+
O(\eta^2),
\]
which proves the first-order expansion.

We next bound its linear term. Since $S_{\mathcal F}$ is differentiable and locally $L_{\mathcal F}$-Lipschitz at $\btheta_0$, for every unit vector $\mathbf v$ and all sufficiently small nonzero $h$,
\[
\frac{
\left|
S_{\mathcal F}(\btheta_0+h\mathbf v)
-
S_{\mathcal F}(\btheta_0)
\right|
}{|h|}
\leq
L_{\mathcal F}.
\]
Taking $h\to0$ and using differentiability yields
\[
\left|
\left\langle
\nabla_{\btheta}S_{\mathcal F}(\btheta_0),
\mathbf v
\right\rangle
\right|
\leq
L_{\mathcal F}.
\]
Taking the supremum over all $\|\mathbf v\|_2=1$ gives
\[
\left\|
\nabla_{\btheta}S_{\mathcal F}(\btheta_0)
\right\|_2
\leq
L_{\mathcal F}.
\]

By the local-separation condition,
\[
\Pi_{\mathcal S}
\nabla_{\btheta}S_{\mathcal F}(\btheta_0)
=
\nabla_{\btheta}S_{\mathcal F}(\btheta_0).
\]
Moreover, because $\Pi_{\mathcal S}$ is an orthogonal projector, it is self-adjoint. Hence,
\begin{align*}
\left\langle
\nabla_{\btheta}S_{\mathcal F}(\btheta_0),
\mathbf g
\right\rangle
&=
\left\langle
\Pi_{\mathcal S}
\nabla_{\btheta}S_{\mathcal F}(\btheta_0),
\mathbf g
\right\rangle \\
&=
\left\langle
\nabla_{\btheta}S_{\mathcal F}(\btheta_0),
\Pi_{\mathcal S}\mathbf g
\right\rangle.
\end{align*}
It follows from the Cauchy--Schwarz inequality that
\begin{align*}
-
\left\langle
\nabla_{\btheta}S_{\mathcal F}(\btheta_0),
\mathbf g
\right\rangle
&=
-
\left\langle
\nabla_{\btheta}S_{\mathcal F}(\btheta_0),
\Pi_{\mathcal S}\mathbf g
\right\rangle \\
&\leq
\left|
\left\langle
\nabla_{\btheta}S_{\mathcal F}(\btheta_0),
\Pi_{\mathcal S}\mathbf g
\right\rangle
\right| \\
&\leq
\left\|
\nabla_{\btheta}S_{\mathcal F}(\btheta_0)
\right\|_2
\left\|
\Pi_{\mathcal S}\mathbf g
\right\|_2 \\
&\leq
L_{\mathcal F}
\left\|
\Pi_{\mathcal S}\mathbf g
\right\|_2.
\end{align*}
Combining this inequality with the explicit remainder estimate gives
\[
\Delta S_{\mathcal F}
\leq
\eta L_{\mathcal F}
\left\|
\Pi_{\mathcal S}\mathbf g_a
\right\|_2
+
\frac{\beta_{\mathcal F}}{2}
\eta^2
\|\mathbf g_a\|_2^2.
\]
For fixed $\mathcal D_a$, the second term is $O(\eta^2)$, and therefore
\[
\Delta S_{\mathcal F}
\leq
\eta L_{\mathcal F}
\left\|
\Pi_{\mathcal S}\mathbf g_a
\right\|_2
+
O(\eta^2),
\]
which completes the proof.
\end{proof}

\subsection{A Simple ReLU Counterexample}
\label{subsec:appx-counterexample}

We give a minimal construction where a zero release score coexists with a nonzero acquisition gradient, while activation suppression reduces that gradient.

Let $\operatorname{ReLU}(z)=\max\{z,0\}$ act element-wise and consider
\[
\mathbf h_u(\xbf)=\operatorname{ReLU}(\mathbf A_u\xbf),
\qquad
f_{\btheta}(\xbf)=\mathbf w^\top\mathbf h_u(\xbf),
\qquad
\mathbf A_u=
\begin{pmatrix}
u&-1\\
-1&1
\end{pmatrix},
\]
where $\mathbf w=(w_f,w_r)^\top$ and $\btheta=(w_f,u,w_r)^\top$.
As illustrated in Fig.~\refs{fig:minimal-relu-counterexample}{a}, the first hidden unit forms the forbidden-sensitive pathway, parameterized by $(w_f,u)$, while the second forms an independent retain pathway controlled by $w_r$.
The fixed entries $-1$ isolate these pathways on the two basis inputs below and keep inactive pre-activations away from the ReLU kink.
Accordingly, let
\[
\mathcal S
=
\{(\alpha,\beta,0)^\top:\alpha,\beta\in\mathbb R\}.
\]

For the forbidden and retain inputs $\xbf_f=\mathbf e_1$ and $\xbf_r=\mathbf e_2$, respectively, any $u>0$ gives
\[
\mathbf h_u(\xbf_f)=(u,0)^\top,
\qquad
\mathbf h_u(\xbf_r)=(0,1)^\top,
\]
and therefore
\[
S_{\mathcal F}(\btheta)
\coloneq
f_{\btheta}(\xbf_f)
=
w_fu,
\qquad
f_{\btheta}(\xbf_r)=w_r.
\]

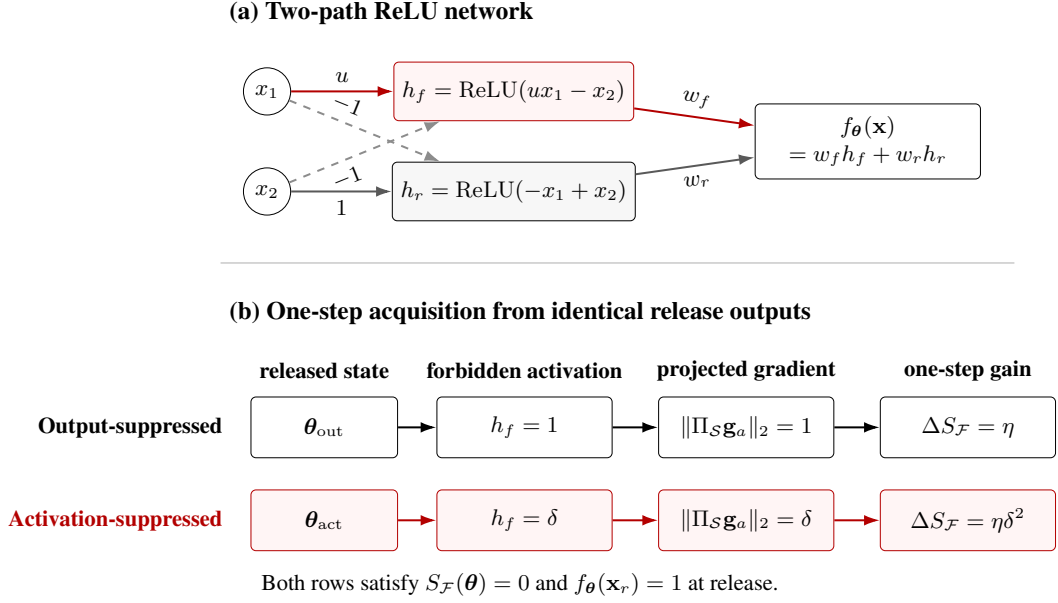
\begin{figure*}[t]
\centering
\begin{adjustbox}{max width=\textwidth}
\begin{tikzpicture}[
    font=\small,
    input/.style={circle,draw,minimum size=6.6mm,inner sep=0pt},
    hidden/.style={draw,rounded corners=2pt,minimum width=31.5mm,minimum height=8.2mm,align=center},
    output/.style={draw,rounded corners=2pt,minimum width=31.5mm,minimum height=10.5mm,align=center},
    cell/.style={draw,rounded corners=2pt,minimum width=24.5mm,minimum height=8.5mm,align=center,inner xsep=1.2mm},
    state/.style={draw,rounded corners=2pt,minimum width=20.5mm,minimum height=8.5mm,align=center},
    flow/.style={-{Latex[length=2.2mm,width=1.5mm]},thick},
    fpath/.style={flow,draw=red!70!black},
    rpath/.style={flow,draw=black!65},
    fixed/.style={flow,dashed,draw=black!45}
]

\node[font=\normalsize\bfseries,anchor=west] at (0,2.15)
    {(a) Two-path ReLU network};

\node[input] (x1) at (0.65,1.05) {$x_1$};
\node[input] (x2) at (0.65,-0.35) {$x_2$};

\node[hidden,draw=red!70!black,fill=red!4]
    (hf) at (4.10,1.05) {$h_f=\operatorname{ReLU}(ux_1-x_2)$};
\node[hidden,fill=black!3]
    (hr) at (4.10,-0.35) {$h_r=\operatorname{ReLU}(-x_1+x_2)$};
\node[output] (out) at (9.05,0.35)
    {$f_{\btheta}(\xbf)$\\[0.3mm]$=w_fh_f+w_rh_r$};

\draw[fpath] (x1) -- node[above] {$u$} (hf);
\draw[fixed] (x2) -- node[pos=.34,below,sloped] {$-1$} (hf);
\draw[fixed] (x1) -- node[pos=.34,above,sloped] {$-1$} (hr);
\draw[rpath] (x2) -- node[below] {$1$} (hr);
\draw[fpath] (hf) -- node[above,sloped] {$w_f$} (out);
\draw[rpath] (hr) -- node[below,sloped] {$w_r$} (out);

\draw[black!18,line width=0.7pt] (0,-1.35) -- (11.1,-1.35);

\node[font=\normalsize\bfseries,anchor=west] at (0,-2.05)
    {(b) One-step acquisition from identical release outputs};
\node[font=\small\bfseries,align=center] at (1.45,-2.85)
    {released state};
\node[font=\small\bfseries,align=center] at (4.25,-2.85)
    {forbidden activation};
\node[font=\small\bfseries,align=center] at (7.35,-2.85)
    {projected gradient};
\node[font=\small\bfseries,align=center] at (10.45,-2.85)
    {one-step gain};

\node[font=\small\bfseries,anchor=east] at (0.15,-3.65) {Output-suppressed};
\node[state] (tout) at (1.45,-3.65) {$\btheta_{\mathrm{out}}$};
\node[cell] (hout) at (4.25,-3.65) {$h_f=1$};
\node[cell] (gout) at (7.35,-3.65) {$\|\Pi_{\mathcal S}\mathbf g_a\|_2=1$};
\node[cell] (dout) at (10.45,-3.65) {$\Delta S_{\mathcal F}=\eta$};

\node[font=\small\bfseries,anchor=east,text=red!70!black] at (0.15,-4.95) {Activation-suppressed};
\node[state,draw=red!70!black,fill=red!4] (tact) at (1.45,-4.95) {$\btheta_{\mathrm{act}}$};
\node[cell,draw=red!70!black,fill=red!4] (hact) at (4.25,-4.95) {$h_f=\delta$};
\node[cell,draw=red!70!black,fill=red!4] (gact) at (7.35,-4.95) {$\|\Pi_{\mathcal S}\mathbf g_a\|_2=\delta$};
\node[cell,draw=red!70!black,fill=red!4] (dact) at (10.45,-4.95) {$\Delta S_{\mathcal F}=\eta\delta^2$};

\draw[flow] (tout.east) -- (hout.west);
\draw[flow] (hout.east) -- (gout.west);
\draw[flow] (gout.east) -- (dout.west);
\draw[fpath] (tact.east) -- (hact.west);
\draw[fpath] (hact.east) -- (gact.west);
\draw[fpath] (gact.east) -- (dact.west);

\node[font=\small,align=center,anchor=west] at (0.45,-5.85)
    {Both rows satisfy $S_{\mathcal F}(\btheta)=0$ and $f_{\btheta}(\xbf_r)=1$ at release.};

\end{tikzpicture}
\end{adjustbox}
\caption{\textbf{A minimal ReLU counterexample.}
Panel (a) shows the forbidden-sensitive path in red and the retain path in gray; dashed arrows are fixed inhibitory connections.
For $\xbf_f=\mathbf e_1$ and $\xbf_r=\mathbf e_2$, both released states in (b) have identical forbidden and retain outputs.
Reducing the forbidden activation from $1$ to $\delta$ reduces the projected attacker-gradient norm by $\delta$ and the exact one-step acquisition gain by $\delta^2$.}
\label{fig:minimal-relu-counterexample}
\end{figure*}

The two rows of Fig.~\refs{fig:minimal-relu-counterexample}{b} trace output suppression and additional activation suppression from their released states to their one-step gains:
\[
\btheta_{\mathrm{out}}=(0,1,1)^\top,
\qquad
\btheta_{\mathrm{act}}=(0,\delta,1)^\top,
\qquad
0<\delta<1.
\]
Both released states satisfy
\[
S_{\mathcal F}(\btheta_{\mathrm{out}})
=
S_{\mathcal F}(\btheta_{\mathrm{act}})
=
0,
\qquad
f_{\btheta_{\mathrm{out}}}(\xbf_r)
=
f_{\btheta_{\mathrm{act}}}(\xbf_r)
=
1,
\]
although their forbidden-sensitive activations are $1$ and $\delta$, respectively.

For this illustrative regression model, the attacker minimizes squared loss, equivalent up to an additive constant to the NLL of a unit-variance Gaussian with mean $f_{\btheta}(\xbf_f)$:
\[
\mathcal L_a(\btheta)
=
\frac12
\left(
f_{\btheta}(\xbf_f)-1
\right)^2
=
\frac12(w_fu-1)^2.
\]
For $u>0$, its gradient is
\[
\mathbf g_a(\btheta)
=
(w_fu-1)
\begin{pmatrix}
u\\w_f\\0
\end{pmatrix}.
\]
Consequently,
\[
\mathbf g_a^{\mathrm{out}}
=
(-1,0,0)^\top,
\qquad
\mathbf g_a^{\mathrm{act}}
=
(-\delta,0,0)^\top.
\]
Both gradients lie in $\mathcal S$, yielding the projected norms shown in the third column of Fig.~\refs{fig:minimal-relu-counterexample}{b}:
\[
\left\|
\Pi_{\mathcal S}\mathbf g_a^{\mathrm{out}}
\right\|_2
=
1,
\qquad
\left\|
\Pi_{\mathcal S}\mathbf g_a^{\mathrm{act}}
\right\|_2
=
\delta.
\]

After one gradient step with the same step size $\eta$,
\[
\btheta_{\mathrm{out}}^+
=
(\eta,1,1)^\top,
\qquad
\btheta_{\mathrm{act}}^+
=
(\eta\delta,\delta,1)^\top,
\]
giving the exact acquisition gains shown in the final column of Fig.~\refs{fig:minimal-relu-counterexample}{b}:
\[
\Delta S_{\mathcal F}^{\mathrm{out}}
=
\eta,
\qquad
\Delta S_{\mathcal F}^{\mathrm{act}}
=
\eta\delta^2.
\]
The first factor of $\delta$ comes from the smaller attacker gradient and the second from the remaining activation; both released models preserve a retain output of one.

Finally, for $u>0$,
\[
\nabla_{\btheta}S_{\mathcal F}(\btheta)
=
(u,w_f,0)^\top
\in\mathcal S,
\]
so the local-separation condition holds.
Both released states lie strictly away from the ReLU kink; in sufficiently small neighborhoods around them, $S_{\mathcal F}(\btheta)=w_fu$ is locally Lipschitz and $1$-smooth.
Thus, this regression construction satisfies the local regularity and separation conditions of Prop.~\ref{prop:future_acquisition_bound}, while showing that identical release-time scores can conceal different acquisition receptivity.
The choice $0<\delta<1$ gives a continuous $\delta^2$ separation while remaining away from the ReLU kink; moving the forbidden gate strictly below zero would additionally set its ReLU derivative to zero, matching the low-slope sealing mechanism that GSU applies during defense optimization.

\section{Pseudocode}
\label{sec:appx-pseudocode}

\begin{algorithm}[H]
\small
\caption{Gradient-Sealed Unlearning (GSU)}
\label{alg:gsu}

\KwIn{
Original parameters $\btheta_o$;
proxy and retain sets $\mathcal D_f,\mathcal D_r$;
look-ahead steps $K_{\mathrm{la}}$;
selection parameters $K,p$;
threshold $\tau<0$;
loss weights $\lambda_s,\lambda_r$.
}
\KwOut{Released parameters $\btheta_{\mathrm{rel}}$.}

\BlankLine
\tcp{Expose}
Fit a disposable copy of $\btheta_o$ on $\mathcal D_f$ for
$K_{\mathrm{la}}$ steps using Eq.~\eqref{eq:gsu-lookahead},
obtaining $\btheta_{\mathrm{la}}$\;

Collect both models' pre-activations under matched
teacher-forcing contexts $(\xbf,\ybf^{<i})$\;

\BlankLine
\tcp{Localize}
Compute $r_{\ell j}$, $m_\ell$, and $R_\ell$ using
Eqs.~\eqref{eq:unit-score} and~\eqref{eq:gsu-layer-recruitment}\;

Select up to $K$ layers with the highest positive $R_\ell$
among those with $m_\ell>0$, forming $\mathcal I^\star$\;

Within each selected layer, select up to $\lceil p d_\ell\rceil$
channels with the highest positive $r_{\ell j}$,
forming $\mathcal N_\ell^\star$\;

$\mathcal G\leftarrow
\{(\ell,j):\ell\in\mathcal I^\star,\ j\in\mathcal N_\ell^\star\}$\;

\BlankLine
\tcp{Seal}
Fix $\mathcal G$, discard $\btheta_{\mathrm{la}}$,
and reset $\btheta\leftarrow\btheta_o$\;

\For{each defense update}{
    Sample minibatches $\mathcal B_f\subseteq\mathcal D_f$
    and $\mathcal B_r\subseteq\mathcal D_r$\;

    Compute the joint loss in Eq.~\eqref{eq:gsu-objective}
    on $(\mathcal B_f,\mathcal B_r)$,
    with $\mathcal L_{\mathrm{seal}}=0$ if $\mathcal G=\varnothing$\;

    Update all parameters $\btheta$ with a first-order optimizer step\;
}

\Return{$\btheta_{\mathrm{rel}}\leftarrow\btheta$}\;
\end{algorithm}

\section{Further Experimental Setup}
\label{sec:appx-exp}

\subsection{Mechanism Study for Figure~\ref{fig:empirical-gradient-evidence}}
\label{subsec:appx-fig2-protocol}

\textbf{Reference, splits, and metrics.}
This mechanism study uses the official \texttt{Qwen3.5-2B} snapshot and a retain-only Reference trained for five epochs on all 3,600 \texttt{retain90} examples (AdamW, peak learning rate $10^{-5}$, seed~0).
The Reference receives no optimizer exposure to $\mathcal D_f$, $\mathcal D_a$, or holdout.
Within each of the 20 target authors, a fixed $4/8/8$ split gives 80 proxy, 160 attacker, and 160 holdout facts; the splits have no exact row overlap but share authors and domain structure.
Forbidden capability is evaluated only on $\mathcal D_a$ using \textsf{ES}.

\textbf{Panel~(a):}
SimNPO, SatImp, NPO, and RMU start from the same Reference and run for 10 release epochs (50 optimizer steps, learning rate $10^{-5}$, effective batch size 16, AdamW, weight decay $0.01$, seed~0) using raw checkpoints without UWC or interpolation.
The acquisition learning rate is calibrated once on the Reference and then applied unchanged to every plotted release state: 10 epochs/100 steps, constant learning rate $4\times10^{-6}$, effective batch size 16, and seed~0, with no early stopping or best-epoch selection.

\textbf{Panels~(b,c):}
For Panel~(b), a disposable Reference copy receives one epoch (10 steps) of $\mathcal D_f$ supervised look-ahead at learning rate $5\times10^{-6}$ and effective batch size 8.
We record answer-token mean \texttt{gate\_proj} pre-activations before and after look-ahead, discard the copy, reload the unchanged Reference, and independently acquire on $\mathcal D_a$ for 10 epochs at learning rate $10^{-6}$ and effective batch size 16.
Both sensitivity arrays are converted to within-layer ranks before pooling.
Panel~(c) uses the same $\mathcal D_a$ schedule (100 steps, seed~0, weight decay 0) and nested layer-quota masks at $0.5/1/2/5/10\%$ coverage.
For each selected channel, attenuation scales the corresponding \texttt{gate\_proj}/\texttt{up\_proj} rows, \texttt{down\_proj} column, and post-Adam update slice by $1-a$; every attenuated arm is matched step-wise to Full-FT in whole-model update norm.
The grid comprises one Full-FT control and 20 localized arms; the 0\% column repeats that control.
Let $G_{p,a}$ be the 10-epoch increase in $F$ for coverage $p$ and attenuation $a$. Each cell reports $R_{p,a}=100(G_{\mathrm{FT}}-G_{p,a})/G_{\mathrm{FT}}$, the percentage reduction in acquisition gain relative to Full-FT.

\subsection{Benchmark Construction}
\label{subsec:appx-benchmarks}

\textbf{\tofu.}
We use the official \texttt{forget10}/\texttt{retain90} partition: 20 forbidden authors with 400 QA pairs and 180 retain authors with 3,600 QA pairs.
A fixed seed-0 ordering partitions each forbidden author's 20 QA pairs equally between $B_1$ and $B_2$, yielding 200 examples per pool.
In the disjoint setting, $(\mathcal D_f,\mathcal D_a)=(B_1,B_2)$; in the identical setting, $(\mathcal D_f,\mathcal D_a)=(B_2,B_2)$.
Thus, attack training and \textsf{ES} evaluation use the same $B_2$ across settings, while disjoint defense uses non-overlapping QA records about the same authors.
The disjoint split does not establish unseen-author or semantically disjoint-fact generalization.
All 3,600 retain examples form $\mathcal D_r$.

\textbf{\wmdp.}
We split each domain's forget corpus in source order: the first $\lfloor N/2\rfloor$ documents form the defender pool $B_1$, and the remainder form the attacker pool $B_2$.
The Bio pools contain 12,226 and 12,227 documents; the Cyber pools contain 500 documents each.
We evaluate only disjoint attacks, with $(\mathcal D_f,\mathcal D_a)=(B_1,B_2)$.
Since the target knowledge is already present, identical defense and attack data would reduce this setting to conventional robust unlearning.
Defense uses the full corresponding retain corpus, containing 60,887 Bio or 4,473 Cyber documents.
Each pool is independently concatenated and tokenized into 512-token sequences using the \openunlearning preprocessing pipeline.
The official Bio and Cyber test sets contain 1,273 and 1,987 multiple-choice questions, respectively; utility is evaluated on all 14,042 \texttt{MMLU} test questions across 57 subjects.

\subsection{Evaluation Metrics}
\label{subsec:appx-evaluation-metrics}

\textbf{Extraction strength.}
For an answer of $L$ tokens, extraction strength (\textsf{ES}) is the length of its longest suffix whose tokens are all correctly predicted by teacher-forced argmax decoding, divided by $L$.
We average this score over the 200 attacker examples.
Let $\textsf{ES}_t$ denote the score after attack epoch $t$, with $t=0$ denoting release.
The main \tofu table reports
\[
    \overline{\textsf{ES}}_3
    = \frac{1}{3}\sum_{t=1}^{3}\textsf{ES}_t,
    \qquad
    \textsf{ES}_5,
\]
which summarize early acquisition and extraction after five attack epochs.

\textbf{\wmdp accuracy.}
We evaluate domain accuracy and sample-weighted \texttt{MMLU} accuracy at attack steps
$s\in\{25,50,75,100,125,150\}$.
Writing $\textsf{F}_j$ for domain accuracy at step $25j$, we report
\[
    \overline{\textsf{F}}_3
    = \frac{1}{3}\sum_{j=1}^{3}\textsf{F}_j,
    \qquad
    \textsf{F}_5.
\]
Thus, the early mean uses steps 25, 50, and 75, while the fifth-checkpoint score uses step 125.

\textbf{Utility and release calibration.}
\tofu utility is $\operatorname{HM}(\textsf{MU},\textsf{Fluency})$, where \textsf{MU} aggregates answer probability, ROUGE-L recall, and truth ratio over the retain, real-author, and world-fact evaluation sets.
These sets contain 400, 100, and 117 examples, respectively; fluency uses the 200 canonical $B_2$ questions.
\wmdp utility is sample-weighted \texttt{MMLU} accuracy~\citep{hendrycks2021measuring}.
The reference utility $U_{\mathrm{ref}}$ is measured on the shared retain-only base for \tofu and the original checkpoint for \wmdp.
For the main comparisons, Unlearning with Control (UWC)~\citep{wang2025towards} calibrates each final defense checkpoint against its reference checkpoint to satisfy
$U_{\mathrm{rel}}\geq0.95U_{\mathrm{ref}}$ before attack.
Release utility is therefore omitted from the main tables.
The separate mechanism and further-analysis protocols specify their own checkpoint handling in
\S\ref{subsec:appx-fig2-protocol} and \S\ref{sec:appx-add-result}.

\textbf{Usable acquisition.}
For both benchmarks, we report
\[
    \textsf{UA}_{90}
    =
    \max_{\substack{t\in\mathcal T\\
                    U_t\geq0.9U_{\mathrm{ref}}}}
    C_t,
\]
where $C_t$ is the forbidden score and $U_t$ is utility at the same checkpoint.
For \tofu, $C_t=\textsf{ES}_t$ and $\mathcal T=\{0,1,\ldots,10\}$ indexes attack epochs, including release.
For \wmdp, $C_t$ is domain accuracy and $\mathcal T=\{25,50,75,100,125,150\}$ indexes attack steps.
The 90\% threshold determines attack-checkpoint eligibility and is distinct from the 95\% release-calibration requirement.
The reference utility remains fixed throughout the attack, and attack checkpoints are not recalibrated.
All main-table scores are percentages, with lower values indicating stronger protection.

\subsection{Models and Initialization}
\label{subsec:appx-models-baselines}

The \tofu checkpoints are \texttt{Llama-3.2-1B-Instruct}, \texttt{Llama-3.2-3B-Instruct}, \texttt{Llama-3.1-8B-Instruct}, and \texttt{Qwen3.5-2B/4B/9B}.
For each model, we construct one retain-only base using five epochs on \texttt{retain90} (1,125 updates), AdamW, learning rate $10^{-5}$, batch size 16, weight decay $0.01$, and a one-epoch warmup followed by linear decay.
This base is shared across methods and both settings.
\wmdp defenses share the original \texttt{zephyr-7b-beta} checkpoint.

\subsection{Common Defense and Attack Configuration}
\label{subsec:appx-defense-optimization}

The following settings describe the main comparisons unless otherwise specified.
All runs use seed~0, bfloat16, a maximum sequence length of 512, and effective batch size 16, with gradient accumulation as needed.
\tofu defense and attack each use ten epochs over 200 examples, giving 130 optimizer updates with the final partial batch retained.
Defense uses a base learning rate of $10^{-5}$ and draws paired retain batches from $\mathcal D_r$.
Both stages use AdamW with weight decay $0.01$, a constant learning rate, and no warmup.
The fixed attack learning rates for \texttt{Llama-3.2-1B/3B} and \texttt{Llama-3.1-8B} are $(10^{-5},10^{-5},5\times10^{-6})$, and those for \texttt{Qwen3.5-2B/4B/9B} are $(10^{-6},5\times10^{-6},10^{-5})$.
We evaluate the release checkpoint and every attack epoch.

\wmdp uses paged 32-bit AdamW with zero weight decay.
Defense runs for 80 updates at learning rate $4\times10^{-6}$, with 16 warmup updates followed by a constant rate; the final defense checkpoint is UWC-calibrated before release.
Attack runs for 150 updates on the attacker corpus at constant learning rate $4\times10^{-6}$ without warmup, with evaluation every 25 updates.


\section{Additional Results}
\label{sec:appx-add-result}

\textbf{Analysis roadmap.}
We connect acquisition performance to component contributions, selected-gate responses,
and benign adaptability.
\S\ref{subsec:appx-gsu-attack-utility} first places the extraction advantage alongside utility.
\S\ref{subsec:appx-gsu-complete-ablation} and \S\ref{subsec:appx-gsu-seal-ablation}
then examine the component design and sealing trajectories.
\S\ref{subsec:appx-gsu-gates-dynamics} and \S\ref{subsec:appx-gsu-postattack-gates}
probe the corresponding gate responses at release and after attack.
Finally, \S\ref{subsec:appx-gsu-benign} and \S\ref{subsec:appx-gsu-benign-attack}
evaluate benign learning and the protection retained after that adaptation.

\textbf{Shared analysis protocol.}
The analyses in Fig.~\ref{fig:gsu-further-analyses} and this section use disjoint \tofu,
with $\mathcal D_f=B_1$ and $\mathcal D_a=B_2$.
Each pool contains 200 non-overlapping QA pairs from the same 20 forbidden authors;
the retain set contains 3,600 QA pairs from 180 other authors.
This record-disjoint split does not establish unseen-author or semantically disjoint-fact generalization.
Attack and gate analyses use \texttt{Qwen3.5-2B}; benign fine-tuning additionally
covers \texttt{Llama-3.2-1B/3B}.
All attack trajectories include their pre-attack checkpoint and ten evaluated attack epochs,
using full-parameter fine-tuning on $B_2$, learning rate $10^{-6}$, and effective batch size 16.
All observations are single runs, without multi-seed error bars or significance tests.
Shared controls and different measurements from the same runs are not independent replications.

\textbf{Metrics and recorded fields.}
The definitions in \S\ref{subsec:appx-evaluation-metrics} apply throughout.
\textsf{ES} is read from the stored \texttt{forbidden.components.es} field
and multiplied by 100 for percentage display.
Utility is the \texttt{perf\_r} composite score, not a single task accuracy.
Trajectory figures and their endpoint tables retain the $100\times$ utility display scale;
\S\ref{subsec:appx-gsu-complete-ablation} instead reports raw composite scores.
These scales represent the same utility metric and should not be read as task accuracies.
The trajectory plots and complete ablation do not apply a reference-utility threshold
or compute $\textsf{UA}_{90}$.

\textbf{Frozen configurations.}
For each model, GSU uses the configuration selected by mean \textsf{ES} over
attack epochs 1--5 among completed disjoint candidates when the analysis plan was frozen.
Configurations are not retuned for individual panels; baselines use the analysis plan's defaults.
These fixed configurations can differ from the main table's metric-specific candidate selections,
so the corresponding numerical results need not coincide.
Because selection uses existing $B_2$ attack evaluations, these are development-stage
diagnostics, not independent-test or equal-tuning-budget comparisons.
Tab.~\ref{tab:gsu-analysis-config} lists the frozen settings.
All use a defense learning rate of $10^{-5}$ and 100 look-ahead updates.

\begin{table}[!htbp]
    \centering
    \caption{\textbf{Frozen configurations for the supplementary analyses.}
    ``Seal alpha'' retains the terminology of the experiment records.}
    \label{tab:gsu-analysis-config}
    \small
    \setlength{\tabcolsep}{8pt}
    \begin{tabular}{@{}lcccc@{}}
        \toprule
        \textbf{Model} & $K$ & $p$ & Seal alpha & $\tau$ \\
        \midrule
        \texttt{Llama-3.2-1B} & 4 & 5\% & 0.01 & $-9$ \\
        \texttt{Qwen3.5-2B}   & 4 & 5\% & 0.03 & $-4$ \\
        \texttt{Llama-3.2-3B} & 2 & 5\% & 0.01 & $-4$ \\
        \bottomrule
    \end{tabular}
\end{table}


\subsection{Utility during Acquisition}
\label{subsec:appx-gsu-attack-utility}

\textbf{Evaluating extraction alongside utility.}
We first ask whether the extraction advantage in Fig.~\refs{fig:gsu-further-analyses}{a}
is accompanied by retained utility.
Fig.~\ref{fig:gsu-attack-utility} adds utility at the same checkpoints to the existing
five-method \textsf{ES} trajectories; it is a complementary measurement of those runs.
All eleven observed checkpoints are retained without smoothing, utility filtering,
or checkpoint selection.

\begin{figure}[!htbp]
    \centering
    \includegraphics[width=\textwidth]{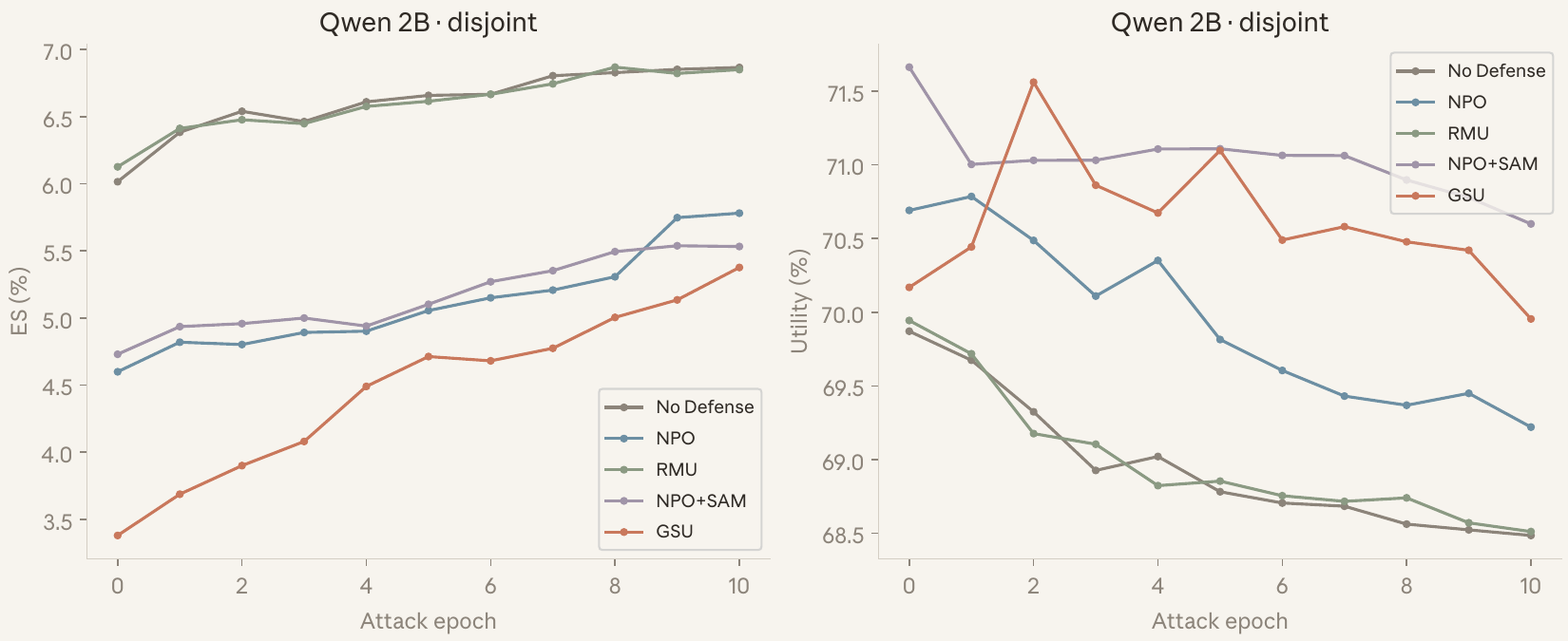}
    \caption{\textbf{Utility accompanying the acquisition trajectory.}
Left: the \textsf{ES} observations from Fig.~\refs{fig:gsu-further-analyses}{a},
with GSU below all displayed controls at every checkpoint.
Right: utility at the same checkpoints on the $100\times$ composite-score scale,
without utility filtering.}
    \label{fig:gsu-attack-utility}
\end{figure}

\textbf{Lower extraction with retained utility.}
GSU maintains lower \textsf{ES} than all four displayed controls throughout the measured attack,
while its utility stays above No Defense and RMU.
At epoch~10, these three methods have utility scores of 69.955, 68.485, and 68.511
on the plotted scale, respectively, as reported in Tab.~\ref{tab:gsu-dynamics-summary}.
This supports an acquisition advantage that is not explained by uniformly lower measured utility.
Relative to NPO+SAM, GSU trades slightly lower final utility (69.955 versus 70.599)
for lower final \textsf{ES} (5.377\% versus 5.533\%); it does not dominate both metrics.
GSU also begins with lower extraction: its \textsf{ES} increases by 1.997 percentage
points over the attack, versus 0.803 for NPO+SAM.
The evidence concerns lower attainable extraction within this budget, not a slower
acquisition rate or a $\textsf{UA}_{90}$ ranking.
We next examine which components contribute to this behavior.

\begin{table}[!htbp]
    \centering
    \caption{\textbf{Extraction and utility during the shared attack.}
\textsf{ES} entries are percentages; $\overline{\textsf{ES}}_3$ averages epochs 1--3.
$\textsf{U}$ is shown on the $100\times$ composite-score scale, not as task accuracy.}
    \label{tab:gsu-dynamics-summary}
    \small
    \setlength{\tabcolsep}{5pt}
    \begin{tabular}{@{}lcccccc@{}}
        \toprule
        \textbf{Method}
        & $\textsf{ES}_0$ & $\overline{\textsf{ES}}_3$
        & $\textsf{ES}_5$ & $\textsf{ES}_{10}$
        & $\textsf{U}_0$ & $\textsf{U}_{10}$ \\
        \midrule
        No Defense & 6.016 & 6.463 & 6.659 & 6.868 & 69.871 & 68.485 \\
        NPO & 4.599 & 4.838 & 5.056 & 5.782 & 70.691 & 69.220 \\
        RMU & 6.127 & 6.446 & 6.616 & 6.851 & 69.944 & 68.511 \\
        NPO+SAM & 4.730 & 4.965 & 5.102 & 5.533 & 71.663 & 70.599 \\
        GSU & 3.379 & 3.889 & 4.713 & 5.377 & 70.169 & 69.955 \\
        \bottomrule
    \end{tabular}
\end{table}

\subsection{Complete Component Ablation}
\label{subsec:appx-gsu-complete-ablation}

\textbf{Setup and metrics.}
To examine the design behind the acquisition advantage, we compare five completed
variants under the frozen \texttt{Qwen3.5-2B}, disjoint \tofu, seed-0 setting:
full GSU, without sealing, without suppression, without retention, and random gates.
The variants follow the component definitions in \S\ref{sec:method}.
Their GSU and without-seal controls are shared with the trajectory analysis in
\S\ref{subsec:appx-gsu-seal-ablation}, rather than independently repeated.
Fig.~\ref{fig:gsu-complete-ablation} reports release utility and
$\overline{\textsf{ES}}_5=\frac{1}{5}\sum_{t=1}^{5}\textsf{ES}_t$.
This early-attack mean excludes epoch~0 and differs from both the main table's
three-epoch mean and $\textsf{UA}_{90}$.
Utility is the raw \texttt{perf\_r} composite score, not an accuracy or a percentage;
Tab.~\ref{tab:gsu-complete-ablation} also reports release and final-attack endpoints.

\begin{figure}[!htbp]
    \centering
    \includegraphics[width=\textwidth]{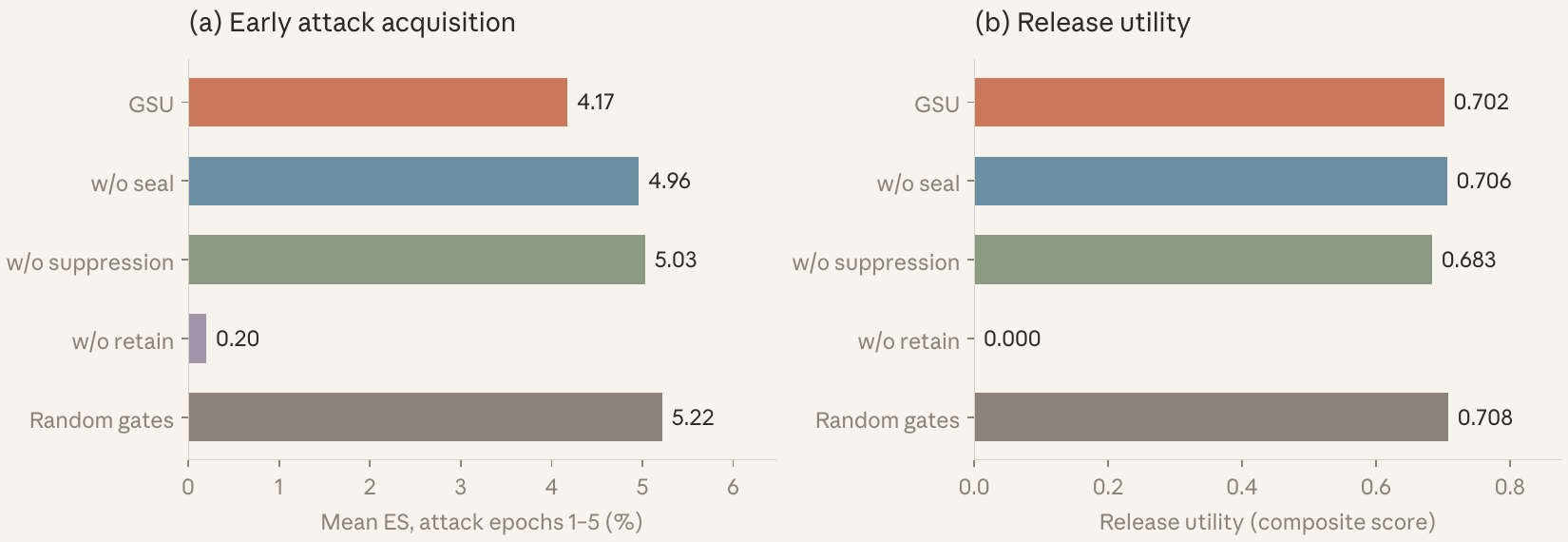}
    \caption{\textbf{Complete component ablation on \texttt{Qwen3.5-2B} under disjoint \tofu.}
Left: mean \textsf{ES} over attack epochs 1--5, in percent.
Right: release utility as a raw composite score.
Full GSU lowers early extraction relative to the three nonzero-utility ablations;
without retention, utility is zero, so low \textsf{ES} alone is not evidence of a
successful utility-preserving defense. Results use one training seed.}
    \label{fig:gsu-complete-ablation}
\end{figure}

\textbf{Complementary component contributions.}
Full GSU achieves a mean \textsf{ES} of 4.1740\%, compared with 4.9598\% without
sealing, 5.0298\% without suppression, and 5.2196\% with random gates.
The reductions of 0.786, 0.856, and 1.046 percentage points support the contributions
of suppression, sealing, and the frozen gate selector to lower early acquisition.
GSU also has lower \textsf{ES} than these three variants at epoch~10, although the gaps narrow.
Its release utility is comparable to, but slightly below, without-seal and random-gate
utility (0.701687 versus 0.705906 and 0.707940), and above without suppression (0.683498).
These results support lower early acquisition at comparable measured utility,
not strict dominance across every metric; without suppression has higher final utility.

\begin{table}[!htbp]
    \centering
    \caption{\textbf{Complete ablation measurements.}
    \textsf{ES} entries are percentages; $\textsf{U}_0$ and $\textsf{U}_{10}$
    are unscaled composite utility scores, not accuracies.
    $\overline{\textsf{ES}}_5$ averages attack epochs 1--5.}
    \label{tab:gsu-complete-ablation}
    \small
    \setlength{\tabcolsep}{5pt}
    \begin{tabular}{@{}lccccc@{}}
        \toprule
        \textbf{Variant} & $\textsf{ES}_0$ & $\overline{\textsf{ES}}_5$
        & $\textsf{ES}_{10}$ & $\textsf{U}_0$ & $\textsf{U}_{10}$ \\
        \midrule
        GSU & 3.3793 & 4.1740 & 5.3768 & 0.701687 & 0.699545 \\
        Without seal & 4.6104 & 4.9598 & 5.7530 & 0.705906 & 0.695012 \\
        Without suppression & 4.1924 & 5.0298 & 5.5130 & 0.683498 & 0.705820 \\
        Without retain & 0.0000 & 0.1997 & 0.7997 & 0.000000 & 0.000000 \\
        Random gates & 4.3565 & 5.2196 & 5.7881 & 0.707940 & 0.695478 \\
        \bottomrule
    \end{tabular}
\end{table}

\textbf{Why retention matters.}
Removing retention yields utility of zero at both endpoints.
Its low \textsf{ES} therefore illustrates the need to assess acquisition suppression
jointly with utility, rather than an advantage over full GSU.
Zero composite utility does not imply that every individual capability is absent.
The random-gate comparison supports the selected configuration but uses only one
random realization, not repeated-sampling evidence.
The component evidence is specific to this single-seed setting; it does not establish
statistical significance or transfer to other models, identical data, or \wmdp.
The next subsection resolves the sealing comparison across the full attack trajectory.

\subsection{Paired Sealing Ablation}
\label{subsec:appx-gsu-seal-ablation}

\textbf{From component summaries to attack trajectories.}
We now examine when the sealing contribution in
\S\ref{subsec:appx-gsu-complete-ablation} is visible.
Fig.~\ref{fig:gsu-seal-ablation} expands the same GSU and without-seal controls into
release evaluation and ten attack epochs on \texttt{Qwen3.5-2B}.
The shared disjoint protocol is unchanged; this is a temporal view of the paired
comparison, not an additional independent ablation.

\begin{figure}[!htbp]
    \centering
    \includegraphics[width=\textwidth]{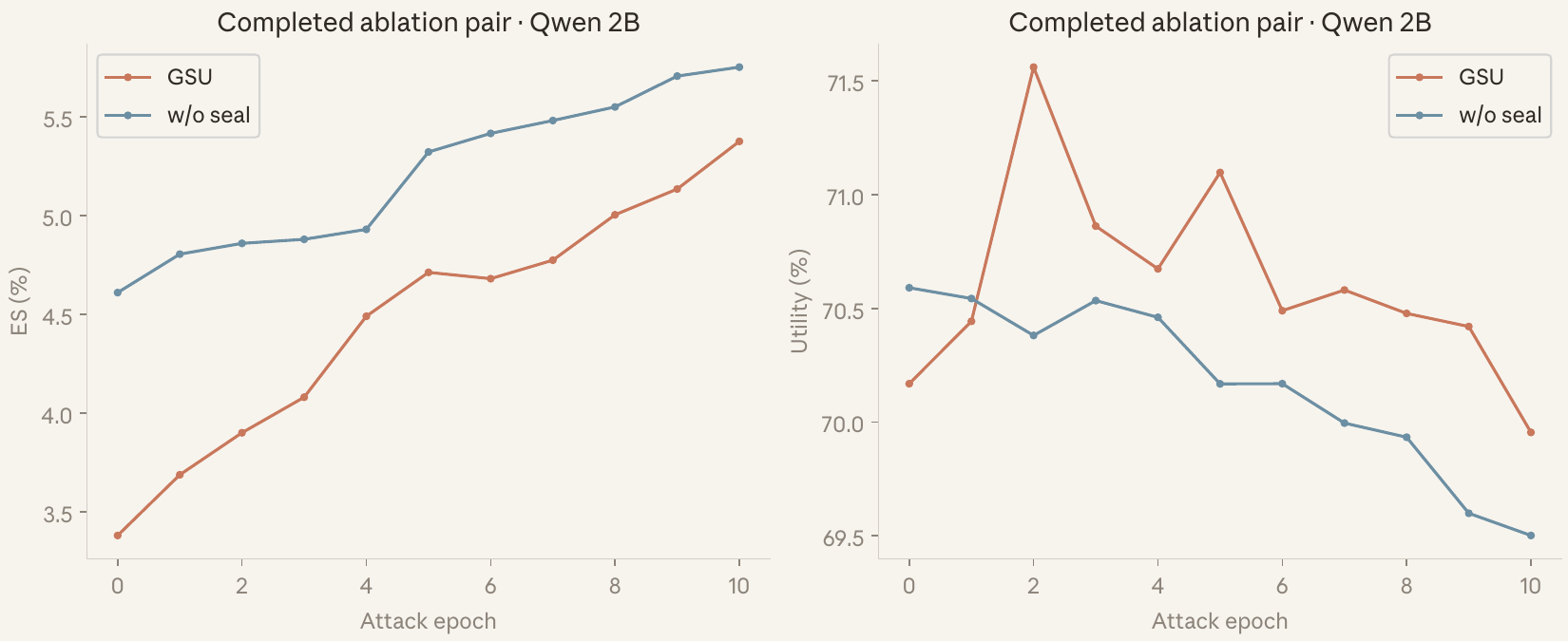}
    \caption{\textbf{Sealing contributions across the attack trajectory.}
The GSU and without-seal runs from \S\ref{subsec:appx-gsu-complete-ablation}
are shown at all eleven observed checkpoints.
GSU retains lower \textsf{ES} (left), while the narrowing gap and accompanying
utility (right, $100\times$ composite-score scale) show how the comparison evolves.}
    \label{fig:gsu-seal-ablation}
\end{figure}

\textbf{A benefit throughout the observed budget.}
GSU and without sealing begin at 3.379\% and 4.610\% \textsf{ES}, respectively,
reach 4.713\% and 5.323\% at epoch~5, and finish at 5.377\% and 5.753\%.
Thus, the sealing advantage remains present at every observed checkpoint,
including a 0.376-percentage-point gap after ten attack epochs.
Final utility is also slightly higher for GSU (69.955 versus 69.501 on the plotted scale).
The trajectories complement the early-mean ablation by showing that its benefit
is not confined to one selected checkpoint.
Because GSU already starts lower and the gap narrows, this supports lower extraction
over the measured budget, rather than a smaller acquisition rate or permanent resistance.
We next probe the selected gates for internal changes consistent with this contribution.

\subsection{Release-Stage Gate Distributions and Negative-Tail Occupancy}
\label{subsec:appx-gsu-gates-dynamics}

\textbf{Measurement scope and aggregation.}
To examine the local changes accompanying the sealing benefit, we supplement
Fig.~\refs{fig:gsu-further-analyses}{b} with gate-wise pre-activation distributions
and negative-tail occupancy at release.
Measurements cover Original, released GSU, and the released without-seal variant
on $B_1$, $B_2$, and retain inputs, before any attacker updates.
``Original'' denotes the control checkpoint entering this analysis, not necessarily
original pretrained weights.
Each group contains 1,232 recorded selected gates.
Both $B_1$ and $B_2$ use all 200 examples, while retain uses all 3,600 examples.
For each gate, we first average over answer-token positions within an example,
then weight examples equally.
Writing $\mathcal G_{\mathrm{rec}}$ for the recorded gates and $S$ for an input set,
\begin{align}
    \bar u_{g,S}(\btheta)
    &= \frac{1}{|S|}\sum_{(\xbf,\ybf)\in S}
       \frac{1}{|\ybf|}\sum_{t=1}^{|\ybf|}u_g^t(\btheta),
       \label{eq:appx-gsu-gate-means}\\
    D_S(\btheta)
    &= \frac{1}{|\mathcal G_{\mathrm{rec}}|}
       \sum_{g\in\mathcal G_{\mathrm{rec}}}
       \frac{1}{|S|}\sum_{(\xbf,\ybf)\in S}
       \frac{1}{|\ybf|}\sum_{t=1}^{|\ybf|}
       \left|\phi'\!\left(u_g^t(\btheta)\right)\right|,
       \label{eq:appx-gsu-gate-diagnostics}
\end{align}
where $u_g^t$ is the pre-activation at the corresponding answer-token context and
$\phi'(u)=\sigma(u)+u\sigma(u)(1-\sigma(u))$.
The derivative statistic averages absolute token-level derivatives, not the
derivative evaluated at the mean pre-activation.

\begin{figure}[!htbp]
    \centering
    \includegraphics[width=\textwidth,trim=0 0 0 20bp,clip]{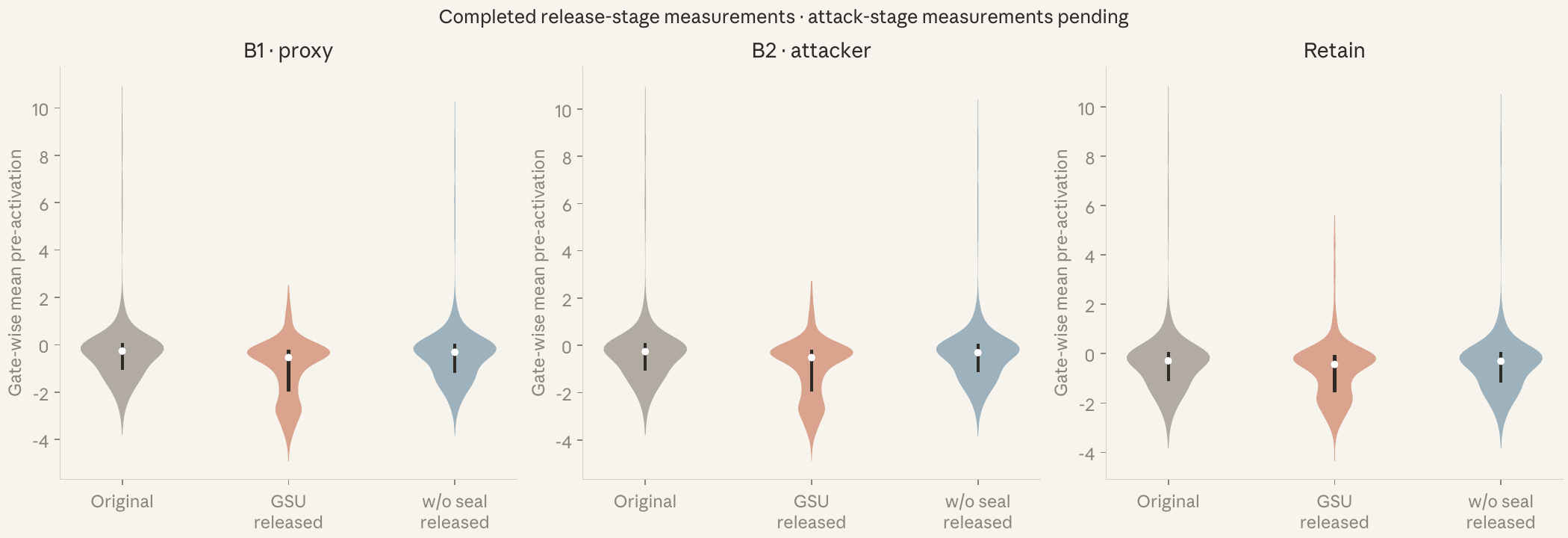}
    \caption{\textbf{Gate-wise pre-activation distributions at release.}
Each violin contains 1,232 gate-wise means, obtained by averaging answer-token
pre-activations within examples and then examples equally.
White markers and dark segments show the median and interquartile range across gates;
the density outlines describe between-gate variation, not uncertainty across training runs.}
    \label{fig:gsu-gate-distributions}
\end{figure}

\begin{samepage}
\textbf{A shift in selected-gate operating regions.}
Fig.~\ref{fig:gsu-gate-distributions} shows a shift toward more negative gate-wise
means under GSU.
On $B_1$, the median moves from $-0.2668$ in Original to $-0.5374$ in GSU;
on $B_2$, it moves from $-0.2656$ to $-0.5159$.
Retain also shifts, from $-0.2877$ to $-0.4358$.
The without-seal medians ($-0.3149$, $-0.3094$, and $-0.3045$, respectively)
remain closer to Original, consistent with a contribution from sealing.
These are distributions of $\bar u_{g,S}$ across gates, not individual token activations.
\par
\end{samepage}

\begin{figure}[!htbp]
    \centering
    \includegraphics[width=\textwidth]{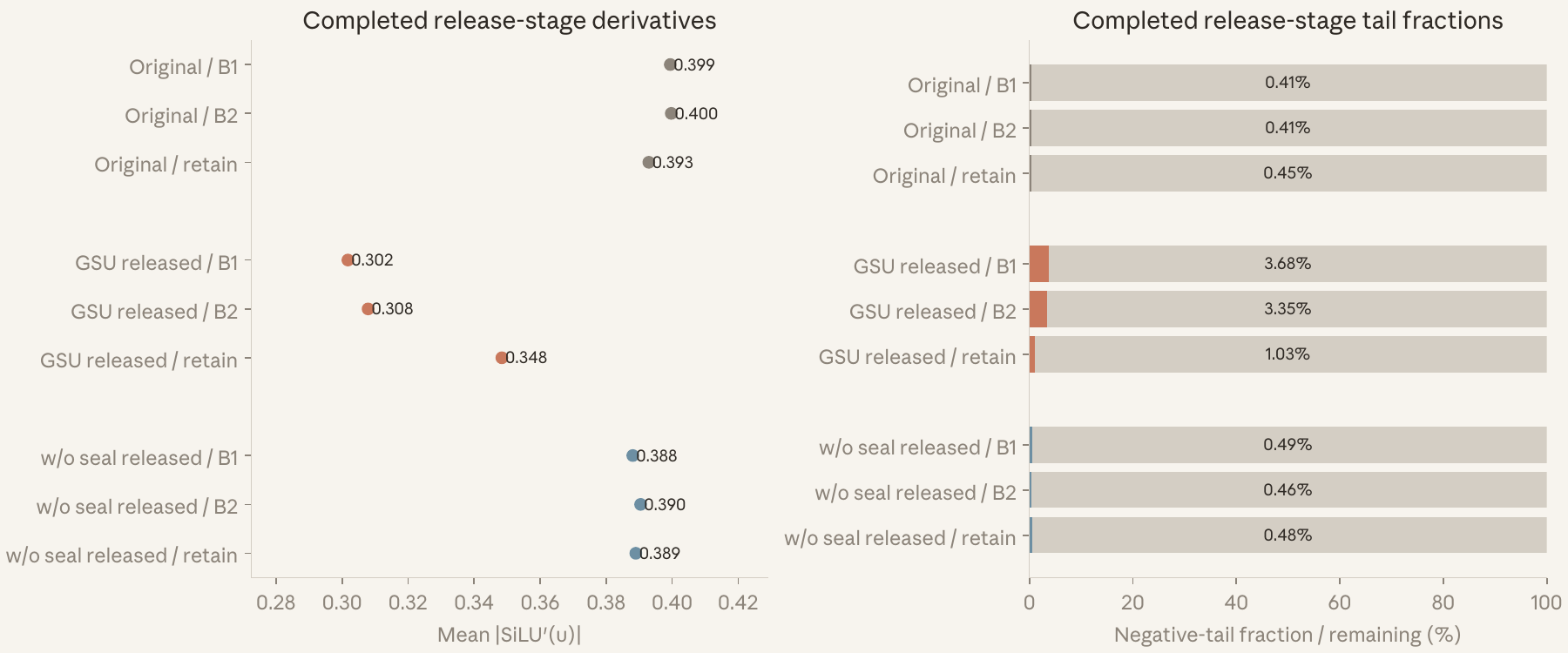}
    \caption{\textbf{Local response and negative-tail occupancy at release.}
Left: mean absolute SiLU derivative, equally averaged across recorded gates.
Right: mean negative-tail fraction at $\tau=-4$ (colored), with the complement in gray.
Both statistics change more on target inputs than retain inputs; they measure
selected-gate responses before attack, not full-network gradients.}
    \label{fig:gsu-gate-tail}
\end{figure}

\textbf{Stronger attenuation on target inputs.}
Fig.~\ref{fig:gsu-gate-tail} and Tab.~\ref{tab:gsu-release-gates} connect this shift
to local derivatives and negative-tail occupancy.
The occupancy is the mean of the recorded gate-wise fractions, multiplied by 100,
using $\tau=-4$ rather than simply $u<0$ and the same answer-token, example, and gate weighting.
GSU raises occupancy on $B_1/B_2$ from approximately 0.41\% to 3.68\%/3.35\%,
versus 0.45\% to 1.03\% on retain; without sealing remains near Original.
The corresponding reduction in $D_S$ is approximately 23--24\% on $B_1/B_2$
and 11\% on retain.
Together, the distributions and local responses are consistent with targeted attenuation
rather than an equally large change across all input sets.
Retain is not unchanged, and the small tail fractions do not imply that most gates or
tokens are saturated.
These release-stage probes are local mechanism evidence, not full-network gradient
measurements or proof of causal isolation.
\S\ref{subsec:appx-gsu-postattack-gates} next checks how the same gate responses change under attack.

\begin{table}[!htbp]
    \centering
    \caption{\textbf{Release-stage gate statistics.}
    Derivatives and negative-tail fractions are equally averaged across recorded gates;
    tail fractions use $\tau=-4$.}
    \label{tab:gsu-release-gates}
    \small
    \setlength{\tabcolsep}{5pt}
    \begin{tabular}{@{}lcccccc@{}}
        \toprule
        & \multicolumn{3}{c}{Mean $|\phi'(u)|$}
        & \multicolumn{3}{c}{Negative-tail fraction (\%)} \\
        \cmidrule(lr){2-4}\cmidrule(lr){5-7}
        \textbf{Inputs}
        & Original & GSU & w/o seal
        & Original & GSU & w/o seal \\
        \midrule
        $B_1$  & 0.3994 & 0.3018 & 0.3881 & 0.411 & 3.684 & 0.487 \\
        $B_2$  & 0.3997 & 0.3079 & 0.3904 & 0.408 & 3.348 & 0.463 \\
        Retain & 0.3929 & 0.3484 & 0.3890 & 0.451 & 1.029 & 0.477 \\
        \bottomrule
    \end{tabular}
\end{table}

\subsection{Selected-Gate Responses before and after Attack}
\label{subsec:appx-gsu-postattack-gates}

\textbf{Matched gate probes.}
We extend the release-only probes in \S\ref{subsec:appx-gsu-gates-dynamics} to
attack epoch~10 on the same frozen \texttt{Qwen3.5-2B} disjoint setting.
Fig.~\ref{fig:gsu-postattack-gates} compares the original base probe state with
GSU and the without-seal variant at release and after attack.
Original is not a No Defense checkpoint after ten attack epochs.
All five states use the same 1,232 selected channels, with 308 channels in each
of four layers, and the 200-example $B_2$ probe set.
The release measurements are shared with the earlier gate analyses, not independent
training repetitions.

\textbf{Measurements.}
We report mean pre-activation, mean absolute SiLU derivative, and mean occupancy
of $u\leq-4$.
For each gate, statistics first average valid answer-prediction positions within
each example, then examples equally, and finally the 1,232 channels equally.
Thus, tail occupancy is a hierarchical average, not the global fraction obtained
by pooling all tokens; only this occupancy is multiplied by 100.
The derivative uses
$|\phi'(u)|=|\sigma(u)+u\sigma(u)(1-\sigma(u))|$
at each valid prediction position, rather than evaluating the derivative at the
mean pre-activation.
Pre-activations here are means, not the gate-wise medians reported in
\S\ref{subsec:appx-gsu-gates-dynamics}.

\begin{figure}[!htbp]
    \centering
    \includegraphics[width=\textwidth]{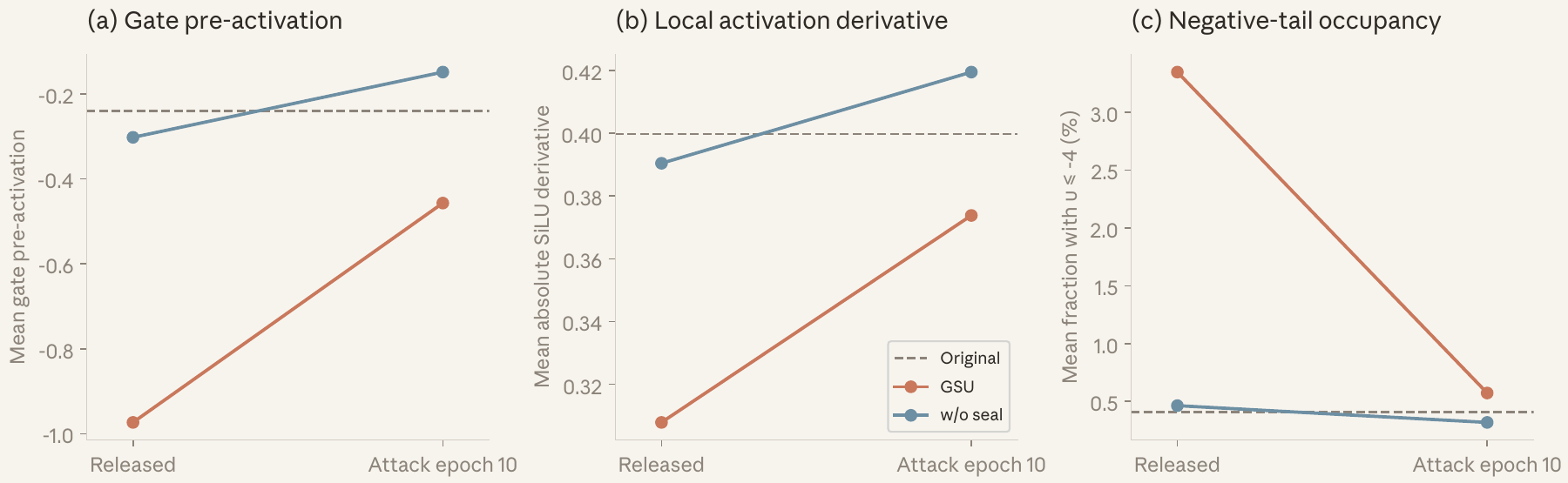}
    \caption{\textbf{Selected-gate responses before and after acquisition attacks.}
GSU and without sealing are probed at release and after ten attack epochs on $B_2$;
dashed lines denote the original base model.
Panels show mean pre-activation, mean absolute SiLU derivative, and occupancy at
$u\leq-4$ over the same 1,232 channels.
Statistics average valid answer-prediction positions within examples, then examples and channels.}
    \label{fig:gsu-postattack-gates}
\end{figure}

\textbf{Residual attenuation after attack.}
After ten attack epochs, GSU retains a lower mean absolute derivative than
without sealing (0.373816 versus 0.419492), more negative mean pre-activations,
and higher negative-tail occupancy, as shown in Tab.~\ref{tab:gsu-postattack-gates}.
This residual difference complements the lower extraction in the paired ablation.
At the same time, attack partially reverses the release-stage changes:
GSU's mean pre-activation moves from $-0.973011$ to $-0.456793$,
its derivative rises from 0.307875 to 0.373816,
and its tail occupancy falls from 3.3482\% to 0.5722\%.
The resulting picture is residual rather than irreversible attenuation:
the selected gates move back toward the original state, but remain different
from the attacked without-seal control.

\begin{table}[!htbp]
    \centering
    \caption{\textbf{Matched $B_2$ gate-response measurements.}
    All states use the same selected channels and aggregation.
    Tail occupancy uses $u\leq-4$ and is reported in percent.}
    \label{tab:gsu-postattack-gates}
    \small
    \setlength{\tabcolsep}{7pt}
    \begin{tabular}{@{}lccc@{}}
        \toprule
        \textbf{Model state} & Mean $u$ & Mean $|\phi'(u)|$ & Tail (\%) \\
        \midrule
        Original base & $-0.240670$ & 0.399727 & 0.4084 \\
        GSU, release & $-0.973011$ & 0.307875 & 3.3482 \\
        GSU, attack epoch 10 & $-0.456793$ & 0.373816 & 0.5722 \\
        Without seal, release & $-0.302069$ & 0.390441 & 0.4628 \\
        Without seal, attack epoch 10 & $-0.148507$ & 0.419492 & 0.3177 \\
        \bottomrule
    \end{tabular}
\end{table}

\begin{samepage}
\textbf{Mechanism interpretation.}
The component controls and matched probes provide complementary empirical support
for sealing: lower measured acquisition is accompanied by lower selected-gate response,
including a residual difference after attack.
The scope remains local. Even at release, only about 3.35\% of valid answer-prediction
positions fall in the negative tail under this aggregation; most gates or tokens
cannot be described as saturated.
Local derivatives do not give whole-network gradients or Hessians and do not directly
measure resistance to optimization; 1,232 channels are not independent training replicates.
These observations support the intended mechanism without a complete causal proof
or a guarantee of irreversible forgetting.
Having examined target-side behavior, we next test whether benign learning remains possible.
\par
\end{samepage}

\subsection{Benign Adaptation: Accuracy and Validation Loss}
\label{subsec:appx-gsu-benign}

\textbf{Setup and metrics.}
We now assess benign adaptability using economics accuracy and full validation-loss
trajectories from the same six runs summarized in Fig.~\refs{fig:gsu-further-analyses}{c}.
Each model has paired No Defense and GSU initializations.
Training uses two epochs of full-parameter AdamW, learning rate $10^{-5}$,
effective batch size 16, weight decay 0.01, bfloat16 precision, maximum sequence
length 512, a constant learning-rate schedule, and seed~0.
The learning rate is used directly without calibration; no defense hooks or
sealing penalties are active during benign training.
Within each model, validation NLL uses the same 128 examples for both initializations;
token counts can differ across model tokenizers.
Economics accuracy is the equally weighted mean of the two \texttt{MMLU} subject
accuracies, \texttt{high\_school\_macroeconomics} and
\texttt{high\_school\_microeconomics}; it is distinct from validation NLL and utility.
Both metrics are evaluated at epochs 0, 1, and 2.

\begin{figure}[!htbp]
    \centering
    \includegraphics[width=\textwidth]{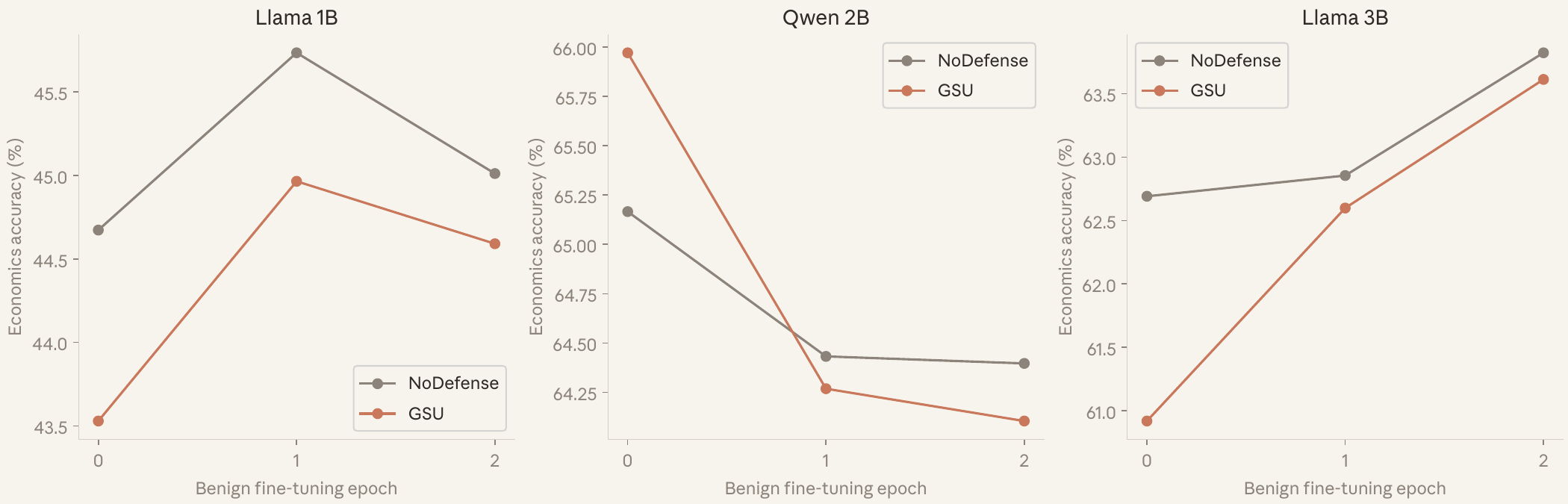}
    \caption{\textbf{Economics accuracy during benign fine-tuning.}
Accuracy is equally averaged over two economics \texttt{MMLU} subjects.
The two Llama GSU models improve and narrow their gaps to No Defense;
\texttt{Qwen3.5-2B} declines for both initializations.
The three checkpoints belong to the same benign runs used for the NLL analysis.}
    \label{fig:gsu-benign-accuracy}
\end{figure}

\textbf{Behavioral adaptation.}
The two Llama GSU models improve their economics accuracy and approach their
undefended controls in Fig.~\ref{fig:gsu-benign-accuracy}.
\texttt{Llama-3.2-1B} rises from 43.529\% to 44.591\%, and
\texttt{Llama-3.2-3B} from 60.918\% to 63.613\%; their gaps to No Defense
shrink from 1.143 to 0.420 and from 1.773 to 0.210 percentage points, respectively.
\texttt{Qwen3.5-2B} instead declines under both GSU (65.971\% to 64.105\%)
and No Defense (65.166\% to 64.397\%).
Thus, the accuracy measurements support useful adaptation in the two Llama models,
while also showing that NLL improvement need not produce accuracy gains on every model.

\begin{figure}[!htbp]
    \centering
    \includegraphics[width=\textwidth]{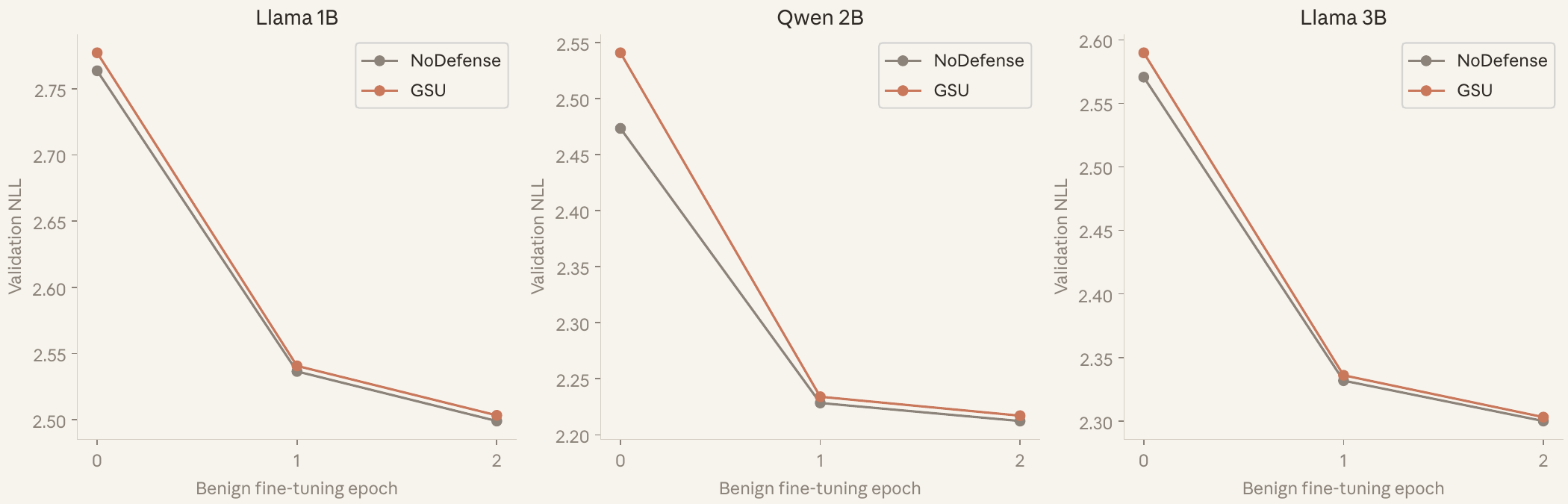}
    \caption{\textbf{Benign validation-loss trajectories.}
NLL at epochs 0, 1, and 2 for the runs summarized in Fig.~\refs{fig:gsu-further-analyses}{c}.
All three GSU models approach the final loss of their undefended controls.
GSU starts and finishes slightly higher, so larger reductions do not imply better
final performance or higher learning efficiency.}
    \label{fig:gsu-benign-validation}
\end{figure}

\begin{samepage}
\textbf{Retained benign optimization ability.}
All three GSU models reduce validation NLL and finish within 0.005 of their
paired No Defense controls, as shown in Fig.~\ref{fig:gsu-benign-validation}
and Tab.~\ref{tab:gsu-benign-gain}.
The table reports the endpoints and
$\Delta\textsf{NLL}=\textsf{NLL}_0-\textsf{NLL}_2$, the absolute reduction in the main figure.
This provides a consistent optimization-level observation across the three models:
normal fine-tuning can bring GSU close to the undefended controls' validation loss.
GSU's larger reductions also reflect its higher initial losses, and its final losses
remain slightly higher; they do not establish superior learning efficiency or
statistical equivalence.
Together with the accuracy measurements, this supports retained benign adaptability
within the tested task, rather than uniformly improved downstream performance.
The final subsection tests whether protection also remains after this adaptation.
\par
\end{samepage}

\begin{table}[!htbp]
    \centering
    \caption{\textbf{Benign validation-loss endpoints.}
    $\Delta\textsf{NLL}$ is the absolute decrease over two epochs;
    displayed values are rounded independently.}
    \label{tab:gsu-benign-gain}
    \small
    \setlength{\tabcolsep}{7pt}
    \begin{tabular}{@{}llccc@{}}
        \toprule
        \textbf{Model} & \textbf{Initialization}
        & $\textsf{NLL}_0$ & $\textsf{NLL}_2$ & $\Delta\textsf{NLL}$ \\
        \midrule
        \multirow{2}{*}{\texttt{Llama-3.2-1B}}
        & No Defense & 2.7635 & 2.4991 & 0.2645 \\
        & GSU        & 2.7771 & 2.5033 & 0.2738 \\
        \midrule
        \multirow{2}{*}{\texttt{Qwen3.5-2B}}
        & No Defense & 2.4733 & 2.2123 & 0.2610 \\
        & GSU        & 2.5408 & 2.2171 & 0.3237 \\
        \midrule
        \multirow{2}{*}{\texttt{Llama-3.2-3B}}
        & No Defense & 2.5707 & 2.2999 & 0.2707 \\
        & GSU        & 2.5899 & 2.3030 & 0.2869 \\
        \bottomrule
    \end{tabular}
\end{table}

\subsection{Resistance after Benign Fine-Tuning}
\label{subsec:appx-gsu-benign-attack}

\textbf{From benign adaptation to renewed attack.}
We finish by testing whether the benign learning in \S\ref{subsec:appx-gsu-benign}
eliminates GSU's protection advantage on \texttt{Qwen3.5-2B}.
No Defense and GSU are compared under two sequences: direct attack, and economics
fine-tuning followed by attack.
The benign stage uses two epochs, learning rate $10^{-5}$, and effective batch size 16,
with the full settings given in \S\ref{subsec:appx-gsu-benign}.
Each resulting model then receives the shared ten-epoch $B_2$ attack.
Fig.~\ref{fig:gsu-benign-then-attack} shows all four trajectories;
attack epoch~0 is immediately before attack, after benign fine-tuning where applicable.
The direct-attack controls are shared with Fig.~\refs{fig:gsu-further-analyses}{a}.

\begin{figure}[H]
    \centering
    \includegraphics[width=\textwidth]{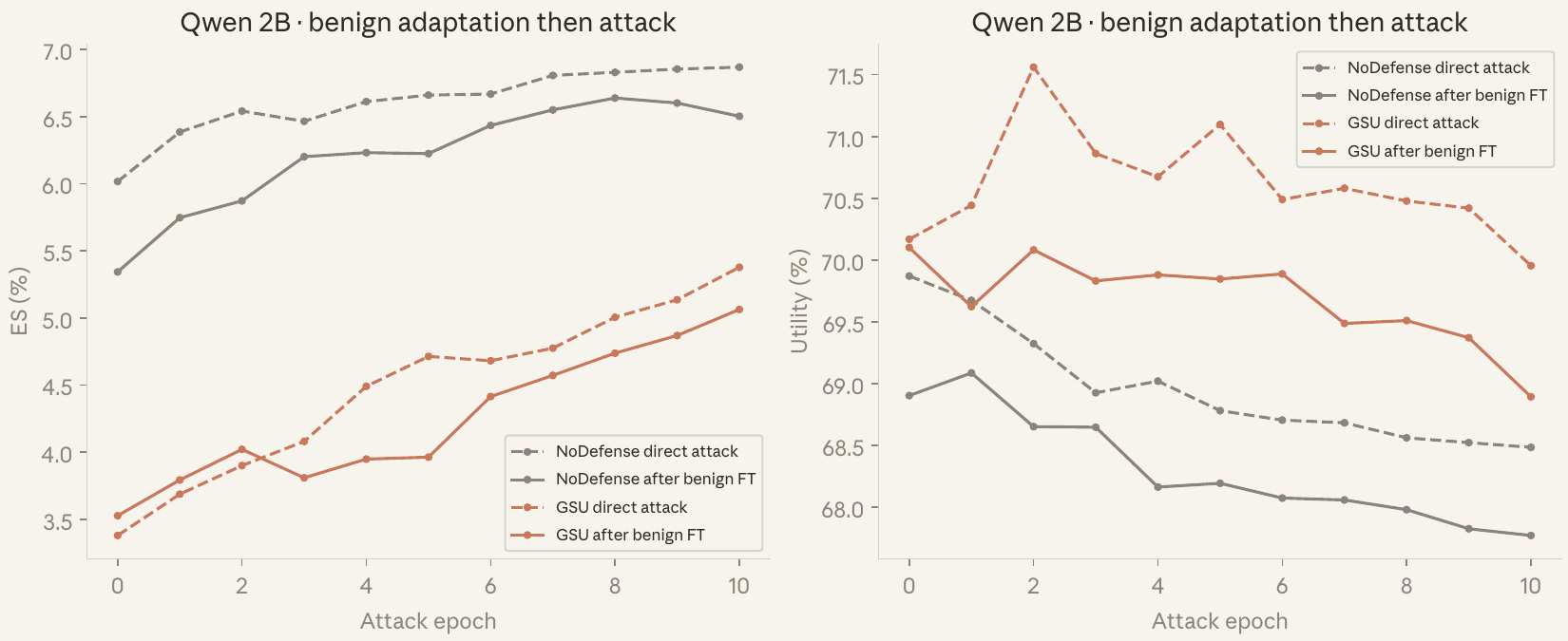}
    \caption{\textbf{Resistance after benign fine-tuning.}
\textsf{ES} (left) and utility (right, $100\times$ composite-score scale) during direct
attacks (dashed) and attacks after two epochs of benign fine-tuning (solid).
GSU retains lower final extraction and higher final utility than the corresponding
undefended control after benign adaptation on \texttt{Qwen3.5-2B}.}
    \label{fig:gsu-benign-then-attack}
\end{figure}

\textbf{Protection retained after adaptation.}
After benign fine-tuning and ten attack epochs, GSU reaches 5.063\% \textsf{ES},
compared with 6.502\% for No Defense, a 1.439-percentage-point advantage.
Its final utility is also higher: 68.894 versus 67.772 on the plotted scale.
Thus, in this experiment, normal benign adaptation does not eliminate GSU's
relative protection advantage.
Prior benign fine-tuning lowers both final \textsf{ES} (5.377\% to 5.063\%)
and utility (69.955 to 68.894) relative to directly attacking GSU,
so adaptation is not a uniform improvement over the unadapted defense.
The result concerns this benign task and attack budget, not arbitrary subsequent training.

\textbf{Summary of the supplementary evidence.}
Across the measured setting, the analyses connect lower forbidden acquisition at
retained utility, contributions from the component design, and selected-gate changes
consistent with sealing.
The benign experiments additionally show that normal learning remains possible
and that a relative protection advantage survives the tested adaptation.
Together, these observations support GSU's goal of limiting target acquisition
without eliminating benign adaptability, within the scope of the reported protocols.

\section{Limitation and Future Work}
\label{sec:appx-limitations}

Our evaluation covers practical downstream fine-tuning rather than unrestricted retraining from scratch with arbitrary compute.
The analysis is local, and GSU relies on the defender proxy set sharing acquisition pathways with future data; broader domain shifts may therefore require richer proxy coverage.
Extending the study to additional architectures, domains, and post-release training procedures is a natural direction for future work.

\end{document}